\documentclass[11pt, a4paper]{deepmind}

\usepackage[authoryear, sort&compress, round]{natbib}

\usepackage{ifthen}
\newboolean{arxiv}
\setboolean{arxiv}{true}
\newboolean{cameraready}
\setboolean{cameraready}{true}
\newboolean{shownotes}
\setboolean{shownotes}{false}

\usepackage{defs}

\title{Explaining grokking through circuit efficiency}

\correspondingauthor{vikrantvarma@deepmind.com, rohinmshah@deepmind.com}

\author[*, 1]{Vikrant Varma}
\author[*, 1]{Rohin Shah}
\author[1]{Zachary Kenton}
\author[1]{János Kramár}
\author[1]{Ramana Kumar}

\affil[*]{Equal contributions}
\affil[1]{Google DeepMind}

\renewcommand{\today}{5 September 2023}

\begin{abstract}
One of the most surprising puzzles in neural network generalisation is \emph{grokking}: a network with perfect training accuracy but poor generalisation will, upon further training, transition to perfect generalisation. We propose that grokking occurs when the task admits a generalising solution and a memorising solution, where the generalising solution is slower to learn but more \emph{efficient}, producing larger logits with the same parameter norm. We hypothesise that memorising circuits become more inefficient with larger training datasets while generalising circuits do not, suggesting there is a critical dataset size at which memorisation and generalisation are equally efficient. We make and confirm four novel predictions about grokking, providing significant evidence in favour of our explanation. Most strikingly, we demonstrate two novel and surprising behaviours: \emph{ungrokking}, in which a network regresses from perfect to low test accuracy, and \emph{semi-grokking}, in which a network shows delayed generalisation to partial rather than perfect test accuracy.
\end{abstract}

\begin{document}
\maketitle

\section{Introduction}

When training a neural network, we expect that once training loss converges to a low value, the network will no longer change much. \citet{power2021grokking} discovered a phenomenon dubbed \emph{grokking} that drastically violates this expectation. The network first ``memorises'' the data, achieving low and stable training loss with poor generalisation, but with further training transitions to perfect generalisation. We are left with the question: \textit{why does the network's test performance improve dramatically upon continued training, having already achieved nearly perfect training performance?}

Recent answers to this question vary widely, including the difficulty of representation learning~\citep{liu2022towards}, the scale of parameters at initialisation~\citep{liu2023omnigrok}, spikes in loss ("slingshots")~\citep{thilak2022slingshot}, random walks among optimal solutions~\citep{millidge2022grokking}, and the simplicity of the generalising solution~\citep[Appendix E]{nanda2023grokking}. In this paper, we argue that the last explanation is correct, by stating a specific theory in this genre, deriving novel predictions from the theory, and confirming the predictions empirically.

We analyse the interplay between the internal mechanisms that the neural network uses to calculate the outputs, which we loosely call ``circuits''~\citep{olah2020zoom}. We hypothesise that there are two families of circuits that both achieve good training performance: one which generalises well (\Gen) and one which memorises the training dataset (\Mem). The key insight is that \emph{when there are multiple circuits that achieve strong training performance, weight decay prefers circuits with high ``efficiency''}, that is, circuits that require less parameter norm to produce a given logit value.

Efficiency answers our question above: if \Gen is more efficient than \Mem, gradient descent can reduce nearly perfect training loss even further by strengthening \Gen while weakening \Mem, which then leads to a transition in test performance. With this understanding, we demonstrate in Section~\ref{sec:theory} that three key properties are sufficient for grokking: (1) \Gen generalises well while \Mem does not, (2) \Gen is more efficient than \Mem, and (3) \Gen is learned more slowly than \Mem.

Since \Gen generalises well, it automatically works for any new data points that are added to the training dataset, and so its efficiency should be independent of the size of the training dataset. In contrast, \Mem must memorise any additional data points added to the training dataset, and so its efficiency should decrease as training dataset size increases. We validate these predictions by quantifying efficiencies for various dataset sizes for both \Mem and \Gen.

This suggests that there exists a crossover point at which \Gen becomes more efficient than \Mem, which we call the critical dataset size $\criticalDatasetSize$. By analysing dynamics at $\criticalDatasetSize$, we predict and demonstrate two new behaviours (Figure~\ref{fig:introduction}). In \emph{ungrokking}, a model that has successfully grokked returns to poor test accuracy when further trained on a dataset much smaller than $\criticalDatasetSize$. In \emph{semi-grokking}, we choose a dataset size where \Gen and \Mem are similarly efficient, leading to a phase transition but only to middling test accuracy.

\begin{figure}[t]
    \centering
    \begin{subfigure}[t]{\textwidth}
        \centering
        \includegraphics[width=0.4\linewidth]{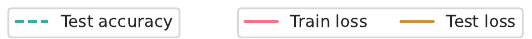}
    \end{subfigure}
    \\
    \begin{subfigure}[t]{0.31\textwidth}
        \centering
        \includegraphics[width=\linewidth]{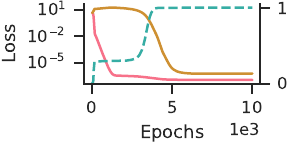}
        \caption{\textbf{Grokking.} The original behaviour from \citet{power2021grokking}. We train on a dataset with $\numTrainExamples \gg \criticalDatasetSize$. As a result, \Gen is strongly preferred to \Mem at convergence, and we observe a transition to very low test loss (100\% test accuracy).}
    \end{subfigure} \hfill
    \begin{subfigure}[t]{0.31\textwidth}
        \centering
        \includegraphics[width=\linewidth]{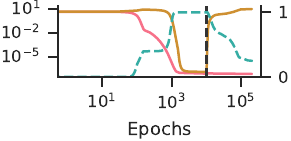}
        \caption{\textbf{Ungrokking.} If we take a network that has already grokked and train it on a new dataset with $\numTrainExamples < \criticalDatasetSize$, the network reverts to significant memorisation, leading to a transition back to poor test loss. (Note the log scale for the x-axis.)}
        \label{fig:ungrokking_accuracy}
    \end{subfigure}  \hfill
    \begin{subfigure}[t]{0.31\textwidth}
        \centering
        \includegraphics[width=\linewidth]{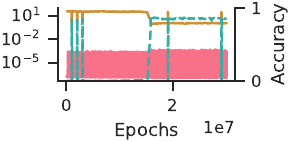}
        \caption{\textbf{Semi-grokking.} When $\numTrainExamples \sim \criticalDatasetSize$, the memorising algorithm and generalising algorithm compete with each other at convergence, so we observe a transition to improved but not perfect test loss (test accuracy of 83\%).}
        \label{fig:semi_grokking_accuracy}
    \end{subfigure}
    \caption{\textbf{Novel grokking phenomena.} When grokking occurs, we expect there are two algorithms that perform well at training: ``memorisation'' (\Mem, with poor test performance) and ``generalisation'' (\Gen, with strong test performance). Weight decay strengthens \Gen over \Mem as the training dataset size increases. By analysing the point at which \Gen and \Mem are equally strong (the critical dataset size $\criticalDatasetSize$), we predict and confirm two novel behaviours: \emph{ungrokking} and \emph{semi-grokking}.}
    \label{fig:introduction}
\end{figure}

We make the following contributions:
\begin{enumerate}
    \item We demonstrate the sufficiency of three ingredients for grokking through a constructed simulation (Section~\ref{sec:theory}).
    \item By analysing dynamics at the ``critical dataset size'' implied by our theory, we \emph{predict} two novel behaviours: \emph{semi-grokking} and \emph{ungrokking} (Section~\ref{sec:dataset-size-conceptual}).
    \item We confirm our predictions through careful experiments, including demonstrating semi-grokking and ungrokking in practice (Section~\ref{sec:grokking-empirics}).
\end{enumerate}

\section{Notation} \label{sec:notation}

We consider classification using deep neural networks under the cross-entropy loss. In particular, we are given a set of inputs $\Inputs$, a set of labels $\Labels$, and a training dataset, $\Dataset = \set{ (\Input{1}, \LabelTrue{1}), \dots (\Input{D}, \LabelTrue{D})}$.

For an arbitrary classifier $\classifier : \Inputs \times \Labels \rightarrow \Reals$, the softmax cross entropy loss is given by:
\begin{equation} \label{eq:x-ent}
    \LossXEnt(\classifier) = -\frac{1}{\numTrainExamples} \sum\limits_{(\Input{}, \LabelTrue{}) \in \Dataset} \log \frac{\exp(\classifier(\Input{}, \LabelTrue{}))}{\sum\limits_{\LabelVar \in \Labels} \exp(\classifier(\Input{}, \LabelVar))}.
\end{equation}
The output of a classifier for a specific class is the class \emph{logit}, denoted by $\logit{\Label{}}{\classifier}(\Input{}) \coloneqq \classifier(\Input{}, \Label{})$. When the input $\Input{}$ is clear from context, we will denote the logit as $\logit{\Label{}}{\classifier}$. We denote the vector of the logits for all classes for a given input as $\logits{\classifier}(\Input{})$ or $\logits{\classifier}$ when $\Input{}$ is clear from context.

Parametric classifiers (such as neural networks) are parameterised with a vector $\theta$ that induces a classifier $\classifier_{\theta}$. The parameter norm of the classifier is $\paramNorm{\classifier_{\theta}} \coloneqq \lVert \theta \rVert$. It is common to add \emph{weight decay} regularisation, which is an additional loss term $\LossWD(\classifier_{\theta}) = \frac{1}{2}(\paramNorm{\classifier_{\theta}})^2$. The overall loss is given by
\begin{equation} \label{eq:loss}
    \Loss(\classifier_{\theta}) = \LossXEnt(\classifier) + \weightDecayHyper \LossWD(\classifier_{\theta}),
\end{equation}
where $\weightDecayHyper$ is a constant that trades off between softmax cross entropy and weight decay.

\paragraph{Circuits.} Inspired by \citet{olah2020zoom}, we use the term \emph{circuit} to refer to an internal mechanism by which a neural network works. We only consider circuits that map inputs to logits, so that a circuit $\circuit{}$ induces a classifier $\classifier_{\circuit{}}$ for the overall task. We elide this distinction and simply write $\circuit{}$ to refer to $\classifier_{\circuit{}}$, so that the logits are $\logit{\Label{}}{\circuit{}}$, the loss is $\Loss(\circuit{})$, and the parameter norm is $\paramNorm{\circuit{}}$.

For any given algorithm, there exist multiple circuits that implement that algorithm. Abusing notation, we use \Gen (\Mem) to refer either to the \emph{family} of circuits that implements the generalising (memorising) algorithm, or a single circuit from the appropriate family.

\section{Three ingredients for grokking} \label{sec:theory}

Given a circuit with perfect training accuracy (as with a pure memorisation approach like \Mem or a perfectly generalising solution like \Gen), the cross entropy loss $\LossXEnt$ incentivises gradient descent to scale up the classifier's logits, as that makes its answers more confident, leading to lower loss (see Theorem~\ref{thm:scaling-logits-decreases-loss}). For typical neural networks, this would be achieved by making the parameters larger. Meanwhile, weight decay $\LossWD$ pushes in the opposite direction, directly decreasing the parameters. These two forces must be balanced at any local minimum of the overall loss.

When we have \emph{multiple} circuits that achieve strong training accuracy, this constraint applies to each individually. But how will they relate to each other? Intuitively, the answer depends on the \emph{efficiency} of each circuit, that is, the extent to which the circuit can convert relatively small parameters into relatively large logits. For more efficient circuits, the $\LossXEnt$ force towards larger parameters is stronger, and the $\LossWD$ force towards smaller parameters is weaker. So, we expect that more efficient circuits will be stronger at any local minimum.

Given this notion of efficiency, we can explain grokking as follows. In the first phase, \Mem is learned quickly, leading to strong train performance and poor test performance. In the second phase, \Gen is now learned, and parameter norm is ``reallocated'' from \Mem to \Gen, eventually leading to a mixture of strong \Gen and weak \Mem, causing an increase in test performance.

This overall explanation relies on the presence of three ingredients:
\begin{enumerate}
    \item \textbf{Generalising circuit:} There are two families of circuits that achieve good training performance: a memorising family \Mem with poor test performance, and a generalising family \Gen with good test performance.
    \item \textbf{Efficiency:} \Gen is more ``efficient'' than \Mem, that is, it can produce equivalent cross-entropy loss on the training set with a lower parameter norm.
    \item \textbf{Slow vs fast learning:} \Gen is learned more slowly than \Mem, such that during early phases of training \Mem is stronger than \Gen.
\end{enumerate}

To illustrate the sufficiency of these ingredients, we construct a minimal example containing all three ingredients, and demonstrate that it leads to grokking. We emphasise that this example is to be treated as a validation of the three ingredients, rather than as a quantitative prediction of the dynamics of existing examples of grokking. Many of the assumptions and design choices were made on the basis of simplicity and analytical tractability, rather than a desire to reflect examples of grokking in practice. The clearest difference is that \Gen and \Mem are modelled as hardcoded input-output lookup tables whose outputs can be strengthened through learned scalar weights, whereas in existing examples of grokking \Gen and \Mem are learned internal mechanisms in a neural network that can be strengthened by scaling up the parameters implementing those mechanisms.

\paragraph{Generalisation.} To model generalisation, we introduce a training dataset $\Dataset$ and a test dataset $\TestDataset$. \Gen is a lookup table that produces logits that achieve perfect train and test accuracy. \Mem is a lookup table that achieves perfect train accuracy, but makes confident incorrect predictions on the test dataset. We denote by $\MemDataset$ the predictions made by \Mem on the test inputs, with the property that there is no overlap between $\TestDataset$ and $\MemDataset$. Then we have:
\begin{align*}
    \logit{\Label{}}{\gen}(\Input{}) &= \indicator{(\Input{}, \Label{}) \in \Dataset\text{ or }(\Input{}, \Label{}) \in \TestDataset} \\
    \logit{\Label{}}{\mem}(\Input{}) &= \indicator{(\Input{}, \Label{}) \in \Dataset\text{ or }(\Input{}, \Label{}) \in \MemDataset}
\end{align*}
\paragraph{Slow vs fast learning.} To model learning, we introduce \emph{weights} for each of the circuits, and use gradient descent to update the weights. Thus, the overall logits are given by:
\begin{equation*}
    \logit{\Label{}}{}(\Input{}) = \weight{\gen}\logit{\Label{}}{\gen}(\Input{}) + \weight{\mem}\logit{\Label{}}{\mem}(\Input{})
\end{equation*}
Unfortunately, if we learn $\weight{\gen}$ and $\weight{\mem}$ directly with gradient descent, we have no control over the \emph{speed} at which the weights are learned. Inspired by \citet{jermyn2022scurves}, we instead compute weights as multiples of two ``subweights'', and then learn the subweights with gradient descent. More precisely, we let $\weight{\gen} = \weight{\gen_1} \weight{\gen_2}$ and $\weight{\mem} = \weight{\mem_1} \weight{\mem_2}$, and update each subweight according to $\weight{i} \leftarrow \weight{i} - \learningRate \cdot \partial{\Loss}/\partial{\weight{i}}$. The speed at which the weights are strengthened by gradient descent can then be controlled by the initial values of the weights. \arxivOnly{Intuitively, the gradient towards the first subweight $\partial{\Loss}/\partial{\weight{1}}$ depends on the strength of the second subweight $\weight{2}$ and vice-versa, and so low initial values lead to slow learning. At initialisation, we set $\weight{\gen_1} = \weight{\mem_1} = 0$ to ensure the logits are initially zero, and then set $\weight{\gen_2} << \weight{\mem_2}$ to ensure \Gen is learned more slowly than \Mem.}

\paragraph{Efficiency.} Above, we operationalised circuit efficiency as the extent to which the circuit can convert relatively small parameters into relatively large logits. When the weights are all one, each circuit produces a one-hot vector as its logits, so their logit scales are the same, and efficiency is determined solely by parameter norm. We define $\paramNorm{\gen}$ and $\paramNorm{\mem}$ to be the parameter norms when weights are all one. Since we want \Gen to be more efficient than \Mem, we set $\paramNorm{\gen} < \paramNorm{\mem}$.

This still leaves the question of how to model parameter norm when the weights are not all one. Intuitively, increasing the weights corresponds to increasing the parameters in a neural network to scale up the resulting outputs. In a $\scalingExp$-layer MLP with Relu activations and without biases, scaling all parameters by a constant $\const$ scales the outputs by $\const^\scalingExp$. Inspired by this observation, we model the parameter norm of $\weight{\gen} \circuit{\gen}$ as $\weight{\gen}^{1/\scalingExp} \paramNorm{i}$ for some $\scalingExp > 0$, and similarly for $\weight{\mem} \circuit{\mem}$.

\begin{figure}[t]
    \centering
    \begin{subfigure}[t]{\textwidth}
        \centering
        \includegraphics[width=0.7\linewidth]{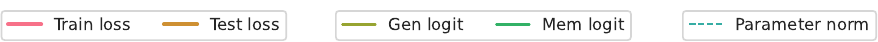}
    \end{subfigure}
    \begin{subfigure}[t]{0.31\textwidth}
        \centering
        \includegraphics[width=\linewidth]{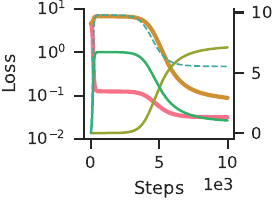}
        \caption{\textbf{All three ingredients.} When \Gen is more efficient than \Mem but learned slower, we observe grokking. \Gen only starts to grow significantly by step 2500, and then substitutes for \Mem. Total parameter norm falls due to \Gen's higher efficiency.
        }
        \label{fig:all-three-ingredients-sim}
    \end{subfigure} \hfill%
    \begin{subfigure}[t]{0.31\textwidth}
        \centering
        \includegraphics[width=\linewidth]{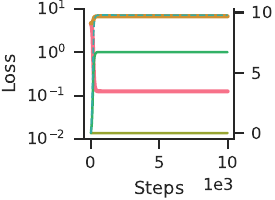}
        \caption{\textbf{\Gen less efficient than \Mem.} We set $\paramNorm{\gen} > \paramNorm{\mem}$. Since \Gen is now less efficient and learned slower, it never grows, and test loss stays high due to \Mem throughout training.}
        \label{fig:not-efficient-sim}
    \end{subfigure} \hfill%
    \begin{subfigure}[t]{0.31\textwidth}
        \centering
        \includegraphics[width=\linewidth]{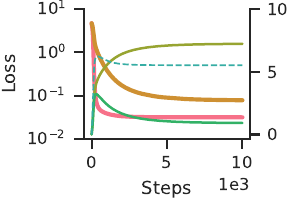}
        \caption{\textbf{\Gen and \Mem learned at equal speeds.} We set $\weight{\gen_1} = \weight{\mem_1}$ so they are learned equally quickly. \Gen is prioritised at least as much as \Mem throughout training, due to its higher efficiency. Thus, test loss is very similar to train loss throughout training, and no grokking is observed.}
        \label{fig:equal-learning-sim}
    \end{subfigure}
    \caption{\textbf{\Gen must be learned slowly for grokking to arise.} We learn weights $\weight{\mem}$ and $\weight{\gen}$ through gradient descent on the loss in Equation~\ref{eq:effective-loss-for-weights}. To model the fact that \Gen is more efficient than \Mem, we set $\paramNorm{\mem} > \paramNorm{\gen}$. We see that we only get grokking when \Gen is learned more slowly than \Mem.}
    \label{fig:sim-result}
\end{figure}

\paragraph{Theoretical analysis.} We first analyse the optimal solutions to the setup above. We can ignore the subweights, as they only affect the speed of learning: $\LossXEnt$ and $\LossWD$ depend only on the weights, not subweights. Intuitively, to get minimal loss, we must assign higher weights to more efficient circuits -- but it is unclear whether we should assign \emph{no} weight to less efficient circuits, or merely smaller but still non-zero weights. Theorem~\ref{thm:efficiency-circuit-weights-logits} shows that in our example, both of these cases can arise: which one we get depends on the value of $\scalingExp$.

\paragraph{Experimental analysis.} We run our example for various hyperparameters, and plot training and test loss in Figure~\ref{fig:sim-result}. We see that when all three ingredients are present (Figure~\ref{fig:all-three-ingredients-sim}), we observe the standard grokking curves, with a delayed decrease in test loss. By contrast, when we make the generalising circuit less efficient (Figure~\ref{fig:not-efficient-sim}), the test loss never falls, and when we remove the slow vs fast learning ingredient (Figure~\ref{fig:equal-learning-sim}), we see that test loss decreases immediately. See Appendix~\ref{app:simulation-details} for details.

\section{Why generalising circuits are more efficient} \label{sec:dataset-size-conceptual}

Section~\ref{sec:theory} demonstrated that grokking can arise when \Gen is more efficient than \Mem, but left open the question of \emph{why} \Gen is more efficient. In this section, we develop a theory based on training dataset size $\numTrainExamples$, and use it to predict two new behaviours: \emph{ungrokking} and \emph{semi-grokking}.

\subsection{Relationship of efficiency with dataset size} \label{sec:efficiency-and-dataset-size}

Consider a classifier $\classifier_{\Dataset}$ obtained by training on a dataset $\Dataset$ of size $\numTrainExamples$ with weight decay, and a classifier $\classifier_{\Dataset'}$ obtained by training on the same dataset with one additional point: $\Dataset' = \Dataset \cup \set{(\Input{}, \LabelTrue{})}$. Intuitively, $\classifier_{\Dataset'}$ cannot be \emph{more} efficient than $\classifier_{\Dataset}$: if it was, then $\classifier_{\Dataset'}$ would outperform $\classifier_{\Dataset}$ even on just $\Dataset$, since it would get similar $\LossXEnt$ while doing better by weight decay. So we should expect that, on average, classifier efficiency is monotonically non-increasing in dataset size.

How does generalisation affect this picture? Let us suppose that $\classifier_{\Dataset}$ successfully generalises to predict $\LabelTrue{}$ for the new input $\Input{}$. Then, as we move from $\Dataset$ to $\Dataset'$, $\LossXEnt(\classifier_{\Dataset})$ likely does not worsen with this new data point. Thus, we could expect to see the same classifier arise, with the same average logit value, parameter norm, and efficiency.

Now suppose $\classifier_{\Dataset}$ instead \emph{fails} to predict the new data point $(\Input{}, \LabelTrue{})$. Then the classifier learned for $\Dataset'$ will likely be \emph{less} efficient: $\LossXEnt(\classifier_{\Dataset})$ would be much higher due to this new data point, and so the new classifier must incur some additional regularisation loss to reduce $\LossXEnt$ on the new point.

Applying this analysis to our circuits, we should expect \Gen's efficiency to remain unchanged as $\numTrainExamples$ increases arbitrarily high, since \Gen does need not to change to accommodate new training examples. In contrast, \Mem must change with nearly every new data point, and so we should expect its efficiency to decrease as $\numTrainExamples$ increases. Thus, when $\numTrainExamples$ is sufficiently large, we expect \Gen to be more efficient than \Mem. (Note however that when the set of possible inputs is small, even the maximal $\numTrainExamples$ may not be ``sufficiently large''.)

\paragraph{Critical threshold for dataset size.} Intuitively, we expect that for extremely small datasets (say, $\numTrainExamples < 5$), it is extremely easy to memorise the training dataset. So, we hypothesise that for these very small datasets, \Mem is more efficient than \Gen. However, as argued above, \Mem will get less efficient as $\numTrainExamples$ increases, and so there will be a critical dataset size $\criticalDatasetSize$ at which \Mem and \Gen are approximately equally efficient. When $\numTrainExamples \gg \criticalDatasetSize$, \Gen is more efficient and we expect grokking, and when $\numTrainExamples \ll \criticalDatasetSize$, \Mem is more efficient and so grokking should not happen.

\paragraph{Effect of weight decay on $\criticalDatasetSize$.} Since $\criticalDatasetSize$ is determined only by the relative efficiencies of \Gen and \Mem, and none of these depends on the exact value of weight decay (just on weight decay being present at all), our theory predicts that $\criticalDatasetSize$ should \emph{not} change as a function of weight decay. Of course, the strength of weight decay may still affect other properties such as the number of epochs till grokking.

\subsection{Implications of crossover: ungrokking and semi-grokking.} \label{sec:ungrokking-semigrokking-conceptual}

By thinking through the behaviour around the critical threshold for dataset size, we predict the existence of two phenomena that, to the best of our knowledge, have not previously been reported.

\paragraph{Ungrokking.} Suppose we take a network that has been trained on a dataset with $\numTrainExamples > \criticalDatasetSize$ and has already exhibited grokking, and continue to train it on a smaller dataset with size $\numTrainExamples' < \criticalDatasetSize$. In this new training setting, \Mem is now more efficient than \Gen, and so we predict that with enough further training gradient descent will reallocate weight from \Gen to \Mem, leading to a transition from high test performance to low test performance. Since this is exactly the opposite observation as in regular grokking, we term this behaviour ``ungrokking''.

Ungrokking can be seen as a special case of catastrophic forgetting~\citep{mccloskey1989catastrophic, ratcliff1990forgetting}, where we can make much more precise predictions. First, since ungrokking should only be expected once $\numTrainExamples' < \criticalDatasetSize$, if we vary $\numTrainExamples'$ we predict that there will be a sharp transition from very strong to near-random test accuracy (around $\criticalDatasetSize$). Second, we predict that ungrokking would arise even if we only remove examples from the training dataset, whereas catastrophic forgetting typically involves training on new examples as well. Third, since $\criticalDatasetSize$ does not depend on weight decay, we predict the amount of ``forgetting'' (i.e. the test accuracy at convergence) also does not depend on weight decay.

\paragraph{Semi-grokking.} Suppose we train a network on a dataset with $\numTrainExamples \approx \criticalDatasetSize$. \Gen and \Mem would be similarly efficient, and there are two possible cases for what we expect to observe (illustrated in Theorem~\ref{thm:efficiency-circuit-weights-logits}).

In the first case, gradient descent would select either \Mem or \Gen, and then make it the maximal circuit. This could happen in a consistent manner (for example, perhaps since \Mem is learned faster it always becomes the maximal circuit), or in a manner dependent on the random initialisation. In either case we would simply observe the presence or absence of grokking.

In the second case, gradient descent would produce a mixture of both \Mem and \Gen. Since neither \Mem nor \Gen would dominate the prediction on the test set, we would expect middling test performance.

\Mem would still be learned faster, and so this would look similar to grokking: an initial phase with good train performance and bad test performance, followed by a transition to significantly improved test performance. Since we only get to middling generalisation unlike in typical grokking, we call this behaviour \emph{semi-grokking}.

Our theory does not say which of the two cases will arise in practice, but in Section~\ref{sec:semigrokking} we find that semi-grokking does happen in our setting.

\section{Experimental evidence} \label{sec:grokking-empirics}

Our explanation of grokking has some support from from prior work: 
\begin{enumerate}
    \item \textbf{Generalising circuit:} \citet[Figure 1]{nanda2023grokking} identify and characterise the generalising circuit learned at the end of grokking in the case of modular addition.
    \item \textbf{Slow vs fast learning:} \citet[Figure 7]{nanda2023grokking} demonstrate ``progress measures'' showing that the generalising circuit develops and strengthens long after the network achieves perfect training accuracy in modular addition.
\end{enumerate}

To further validate our explanation, we empirically test our predictions from Section~\ref{sec:dataset-size-conceptual}:
\begin{enumerate}
    \item[(P1)] \textbf{Efficiency:} We confirm our prediction that \Gen efficiency is independent of dataset size, while \Mem efficiency decreases as training dataset size increases.
    \item[(P2)] \textbf{Ungrokking (phase transition):} We confirm our prediction that ungrokking shows a phase transition around $\criticalDatasetSize$.
    \item[(P3)] \textbf{Ungrokking (weight decay):} We confirm our prediction that the final test accuracy after ungrokking is independent of the strength of weight decay.
    \item[(P4)] \textbf{Semi-grokking:} We demonstrate that semi-grokking occurs in practice.
\end{enumerate}

\paragraph{Training details.} We train 1-layer Transformer models with the AdamW optimiser~\citep{loshchilov2019adamw} on cross-entropy loss (see Appendix~\ref{app:experimental-details} for more details). All results in this section are on the modular addition task ($a + b \mod P$ for $a, b \in (0,\dots,P-1)$ and $P=113$) unless otherwise stated; results on 9 additional tasks can be found in Appendix~\ref{app:experimental-details}.

\subsection{Relationship of efficiency with dataset size} \label{sec:single-circuit-empirical-evidence}

\begin{figure}[t]
    \centering
    \begin{subfigure}[t]{0.48\textwidth}
        \centering
        \includegraphics[width=\linewidth,trim={0 0 8cm 0},clip]{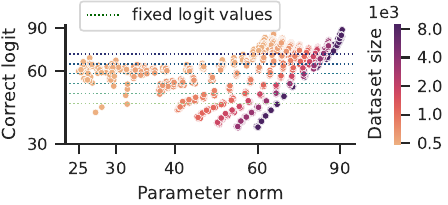}
        \caption{\textbf{\Mem scatter plot.} At a fixed logit value (dotted horizontal lines), parameter norm increases with dataset size.}
        \label{fig:mem-scatter-plot}
    \end{subfigure}  \hfill
    \begin{subfigure}[t]{0.48\textwidth}
        \centering
        \includegraphics[width=\linewidth,trim={8cm 0 0 0},clip]{figure_1_mem_efficiency.pdf}
        \caption{\textbf{\Mem isologit curves.} Curves go up and right, showing that parameter norm increases with dataset size when holding logits fixed.}
        \label{fig:mem-isologits}
    \end{subfigure}
    \begin{subfigure}[t]{0.48\textwidth}
        \centering
        \includegraphics[width=\linewidth,trim={0 0 8cm 0},clip]{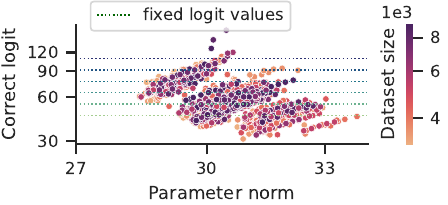}
        \caption{\textbf{\Gen scatter plot.} There is no obvious structure to the colours, suggesting that the logit to parameter norm relationship is independent of dataset size.}
        \label{fig:gen-scatter-plot}
    \end{subfigure}  \hfill
    \begin{subfigure}[t]{0.48\textwidth}
        \centering
        \includegraphics[width=\linewidth,trim={8cm 0 0 0},clip]{figure_1_trig_efficiency.pdf}
        \caption{\textbf{\Gen isologit curves.} The curves are flat, showing that for fixed logit values the parameter norm does not depend on dataset size.}
        \label{fig:gen-isologits}
    \end{subfigure}
    \caption{\textbf{Efficiency of the \Mem and \Gen algorithms.} We collect and visualise a dataset of triples $\left( \logit{\Label{}}{}, \paramNorm{m}, \datasetSize \right)$ (correct logit, parameter norm, and dataset size), each corresponding to a training run with varying random seed, weight decay, and dataset size, for both \Mem and \Gen. Besides a standard scatter plot, we geometrically bucket logit values into six buckets, and plot ``isologit curves'' showing the dependence of parameter norm on dataset size for each bucket. The results validate our theory that (1) \Mem requires larger parameter norm to produce the same logits as dataset size increases, and (2) \Gen uses the same parameter norm to produce fixed logits, irrespective of dataset size. In addition, \Mem has a much wider range of parameter norms than \Gen, and at the extremes can be more efficient than \Gen.}
    \label{fig:efficiencies-mem-gen}
\end{figure}

We first test our prediction about memorisation and generalisation efficiency:

\paragraph{(P1) Efficiency.} We predict (Section~\ref{sec:efficiency-and-dataset-size}) that memorisation efficiency decreases with increasing train dataset size, while generalisation efficiency stays constant.

To test (P1), we look at training runs where only one circuit is present, and see how the logits $\logit{\Label{}}{i}$ vary with the parameter norm $\paramNorm{i}$ (by varying the weight decay) and the dataset size $\datasetSize$.

\paragraph{Experiment setup.} We produce \Mem-only networks by using completely random labels for the training data~\citep{zhang2021understanding}, and assume that the entire parameter norm at convergence is allocated to memorisation. We produce \Gen-only networks by training on large dataset sizes and checking that $>95\%$ of the logit norm comes from just the trigonometric subspace (see Appendix~\ref{app:trig-explanation} for details).

\paragraph{Results.} Figures~\ref{fig:mem-scatter-plot} and~\ref{fig:mem-isologits} confirm our theoretical prediction for memorisation efficiency. Specifically, to produce a fixed logit value, a higher parameter norm is required when dataset size is increased, implying decreased efficiency. In addition, for a fixed dataset size, scaling up logits requires scaling up parameter norm, as expected. Figures~\ref{fig:gen-scatter-plot} and~\ref{fig:gen-isologits} confirm our theoretical prediction for generalisation efficiency. To produce a fixed logit value, the same parameter norm is required irrespective of the dataset size.

Note that the figures show significant variance across random seeds. We speculate that there are many different circuits implementing the same overall algorithm, but they have different efficiencies, and the random initialisation determines which one gradient descent finds. For example, in the case of modular addition, the generalising algorithm depends on a set of ``key frequencies''~\citep{nanda2023grokking}; different choices of key frequencies could lead to different efficiencies.

It may appear from Figure~\ref{fig:gen-scatter-plot} that increasing parameter norm does not increase logit value, contradicting our theory. However, this is a statistical artefact caused by the variance from the random seed. We \emph{do} see ``stripes'' of particular colours going up and right: these correspond to runs with the same seed and dataset size, but different weight decay, and they show that when the noise from the random seed is removed, increased parameter norm clearly leads to increased logits.

\subsection{Ungrokking: overfitting after generalisation} \label{sec:ungrokking}

\begin{figure}[t]
    \centering
    \begin{minipage}{.47\textwidth}
        \centering
        \includegraphics[width=\linewidth]{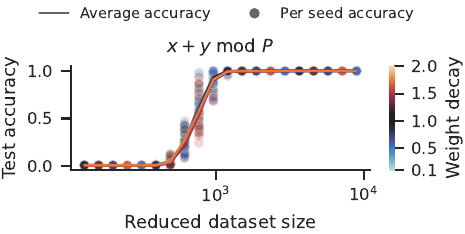}
        \captionof{figure}{\textbf{Ungrokking.} We train on the full dataset (achieving 100\% test accuracy), and then continue training on a smaller subset of the full dataset. We plot test accuracy against reduced dataset size for a range of weight decays. We see a sharp transition from strong test accuracy to near-zero test accuracy, that is independent of weight decay (different coloured lines almost perfectly overlap). See Figure~\ref{fig:ungrokking_many_tasks} for more tasks.}
        \label{fig:ungrokking_addition}
    \end{minipage}\hfill%
    \begin{minipage}{.47\textwidth}
        \centering
        \includegraphics[width=\linewidth]{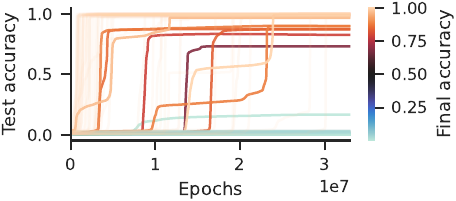}
        \captionof{figure}{\textbf{Semi-grokking.} We plot test accuracy against training epochs for a large sweep of training runs with varying dataset sizes. Lines are coloured by the final test accuracy at the end of training. Out of 200 runs, at least 6 show clear semi-grokking at the end of training. Many other runs show transient semi-grokking, hovering around middling test accuracy for millions of epochs, or having multiple plateaus, before fully generalising.}
        \label{fig:semi_grokking_more_seeds}
    \end{minipage}
\end{figure}

We now turn to testing our predictions about ungrokking. Figure~\ref{fig:ungrokking_accuracy} demonstrates that ungrokking happens in practice. In this section we focus on testing that it has the properties we expect.

\paragraph{(P2) Phase transition.} We predict (Section~\ref{sec:ungrokking-semigrokking-conceptual}) that if we plot test accuracy at convergence against the size of the reduced training dataset $\numTrainExamples'$, there will be a phase transition around $\criticalDatasetSize$.

\paragraph{(P3) Weight decay.} We predict (Section~\ref{sec:ungrokking-semigrokking-conceptual}) that test accuracy at convergence is independent of the strength of weight decay.

\paragraph{Experiment setup.} We train a network to convergence on the full dataset to enable perfect generalisation, then continue training the model on a small subset of the full dataset, and measure the test accuracy at convergence. We vary both the size of the small subset, as well as the strength of the weight decay.

\paragraph{Results.} Figure~\ref{fig:ungrokking_addition} shows the results, and clearly confirms both (P2) and (P3). Appendix~\ref{app:experimental-details} has additional results, and in particular Figure~\ref{fig:ungrokking_many_tasks} replicates the results for many additional tasks.

\subsection{Semi-grokking: evenly matched circuits} \label{sec:semigrokking}

Unlike the previous predictions, semi-grokking is not strictly implied by our theory. However, as we will see, it turns out that it does occur in practice.

\paragraph{(P4) Semi-grokking.} When training at around $\numTrainExamples \approx \criticalDatasetSize$, where \Mem and \Gen have roughly equal efficiencies, the final network at convergence should either be entirely composed of the most efficient circuit, or of roughly equal proportions of \Mem and \Gen. If the latter, we should observe a transition to middling test accuracy well after near-perfect train accuracy.

There are a number of difficulties in demonstrating an example of semi-grokking in practice. First, the time to grok increases super-exponentially as the dataset size $\datasetSize$ decreases~\citep[Figure 1]{power2021grokking}, and $\criticalDatasetSize$ is significantly smaller than the smallest dataset size at which grokking has been demonstrated. Second, the random seed causes significant variance in the efficiency of \Gen and \Mem, which in turn affects $\criticalDatasetSize$ for that run.  Third, accuracy changes sharply with the \Gen to \Mem ratio (Appendix~\ref{app:experimental-details}). To observe a transition to middling accuracy, we need to have balanced \Gen and \Mem outputs, but this is difficult to arrange due to the variance with random seed. To address these challenges, we run many different training runs, on dataset sizes slightly \emph{above} our best estimate of the typical $\criticalDatasetSize$, such that some of the runs will (through random noise) have an unusually inefficient \Gen or an unusually efficient \Mem, such that the efficiencies match and there is a chance to semi-grok.

\paragraph{Experiment setup.} We train 10 seeds for each of 20 dataset sizes evenly spaced in the range $[1500, 2050]$ (somewhat above our estimate of $\criticalDatasetSize$).

\paragraph{Results.} Figure~\ref{fig:semi_grokking_accuracy} shows an example of a single run that demonstrates semi-grokking, and Figure~\ref{fig:semi_grokking_more_seeds} shows test accuracies over time for every run. These validate our initial hypothesis that semi-grokking may be possible, but also raise new questions.

In Figure~\ref{fig:semi_grokking_accuracy}, we see two phenomena peculiar to semi-grokking: (1) test accuracy ``spikes'' several times throughout training before finally converging, and (2) training loss fluctuates in a set range. We leave investigation of these phenomena to future work.

In Figure~\ref{fig:semi_grokking_more_seeds}, we observe that there is often \emph{transient} semi-grokking, where a run hovers around middling test accuracy for millions of epochs, or has multiple plateaus, before generalising perfectly. We speculate that each transition corresponds to gradient descent strengthening a new generalising circuit that is more efficient than any previously strengthened circuit, but took longer to learn. We would guess that if we had trained for longer, many of the semi-grokking runs would exhibit full grokking, and many of the runs that didn't generalise at all would generalise at least partially to show semi-grokking.

Given the difficulty of demonstrating semi-grokking, we only run this experiment on modular addition. However, our experience with modular addition shows that if we only care about values at convergence, we can find them much faster by ungrokking from a grokked network (instead of semi-grokking from a randomly initialised network). Thus the ungrokking results on other tasks (Figure~\ref{fig:ungrokking_many_tasks}) provide some support that we would see semi-grokking on those tasks as well.

\section{Related work}

\paragraph{Grokking.}  Since \citet{power2021grokking} discovered grokking, many works have attempted to understand why it happens. \citet{thilak2022slingshot} suggest that ``slingshots'' could be responsible for grokking, particularly in the absence of weight decay, and \citet{notsawo2023predicting} discuss a similar phenomenon of ``oscillations''. In contrast, our explanation applies even where there are no slingshots or oscillations (as in most of our experiments). \citet{liu2022towards} show that for a specific (non-modular) addition task with an inexpressive architecture, perfect generalisation occurs when there is enough data to determine the appropriate structured representation. However, such a theory does not explain semi-grokking. We argue that for more typical tasks on which grokking occurs, the critical factor is instead the relative efficiencies of \Mem and \Gen. \citet{liu2023omnigrok} show that grokking occurs even in non-algorithmic tasks when parameters are initialised to be very large, because memorisation happens quickly but it takes longer for regularisation to reduce parameter norm to the ``Goldilocks zone'' where generalisation occurs. This observation is consistent with our theory: in non-algorithmic tasks, we expect there exist efficient, generalising circuits, and the increased parameter norm creates the final ingredient (slow learning), leading to grokking. However, we expect that in algorithmic tasks such as the ones we study, the slow learning is caused by some factor other than large parameter norm at initialisation.
    
\citet{davies2023unifying} is the most related work. The authors identify three ingredients for grokking that mirror ours: a \emph{generalising} circuit that is \emph{slowly learned} but is \emph{favoured by inductive biases}, a perspective also mirrored in \citet[Appendix E]{nanda2023grokking}. We operationalise the ``inductive bias'' ingredient as \emph{efficiency} at producing large logits with small parameter norm, and to provide significant empirical support by predicting and verifying the existence of the critical threshold $\criticalDatasetSize$ and the novel behaviours of semi-grokking and ungrokking.
    
\citet{nanda2023grokking} identify the trigonometric algorithm by which the networks solve modular addition after grokking, and show that it grows smoothly over training, and \citet{chughtai2023grokking} extend this result to arbitrary group compositions. We use these results to define metrics for the strength of the different circuits (Appendix~\ref{app:trig-explanation}) which we used in preliminary investigations and for some results in Appendix~\ref{app:experimental-details} (Figures~\ref{fig:grokking_story} and~\ref{fig:trig_to_mem}). \citet{merrill2023tale} show similar results on sparse parity: in particular, they show that a sparse subnetwork is responsible for the well-generalising logits, and that it grows as grokking happens.

\paragraph{Weight decay.} While it is widely known that weight decay can improve generalisation~\citep{krogh1991simple}, the mechanisms for this effect are multifaceted and poorly understood~\citep{zhang2018three}. We propose a mechanism that is, to the best of our knowledge, novel: generalising circuits tend to be more efficient than memorising circuits at large dataset sizes, and so weight decay preferentially strengthens the generalising circuits.

\paragraph{Understanding deep learning through circuit-based analysis.} One goal of \emph{interpretability} is to understand the internal mechanisms by which neural networks exhibit specific behaviours\arxivOnly{, often through the identification and characterisation of specific parts of the network, especially computational subgraphs (``circuits'') that implement human-interpretable algorithms}~\citep{olah2020zoom, elhage2021mathematical, erhan2009visualizing, meng2022locating, cammarata2021curve, wang2022interpretability, li2022emergent, geva2020transformer}. Such work can also be used to understand deep learning.

\citet{olsson2022context} explain a phase change in the training of language models by reference to \emph{induction heads}, a family of circuits that produce in-context learning. In concurrent work, \citet{singh2023transient} show that the in-context learning from induction heads is later replaced by in-weights learning in the absence of weight decay, but remains strong when weight decay is present. We hypothesise that this effect is also explained through circuit efficiency: the in-context learning from induction heads is a generalising algorithm and so is favoured by weight decay given a large enough dataset size.

\citet{michaud2023quantization} propose an explanation for power-law scaling~\citep{hestness2017deep, kaplan2020scaling, barak2022hidden} based on a model in which there are many discrete quanta (algorithms) and larger models learn more of them. Our explanation involves a similar structure: we posit the existence of two algorithms (quanta), and analyse the resulting training dynamics.

\section{Discussion}

\paragraph{Rethinking generalisation.} \citet{zhang2021understanding} pose the question of why deep neural networks achieve good generalisation even when they are easily capable of memorising a random labelling of the training data. Our results gesture at a resolution: in the presence of weight decay, circuits that generalise well are likely to be more efficient given a large enough dataset, and thus are preferred over memorisation circuits, even when both achieve perfect training loss (Section~\ref{sec:efficiency-and-dataset-size}). Similar arguments may hold for other types of regularisation as well.

\paragraph{Necessity of weight decay.} The biggest limitation of our explanation is that it relies crucially on weight decay, but grokking has been observed even when weight decay is not present~\citep{power2021grokking, thilak2022slingshot} (though it is slower and often much harder to elicit~\citep[Appendix D.1]{nanda2023grokking}). This demonstrates that our explanation is incomplete.

Does it also imply that our explanation is \emph{incorrect}? We do not think so. We speculate there is at least one other effect that has a similar regularising effect favouring \Gen over \Mem, such as the implicit regularisation of gradient descent~\citep{soudry2018implicit, lyu2019gradient, wang2021implicit, smith2017bayesian}, and that the speed of the transition from \Mem to \Gen is based on the \emph{sum} of these effects and the effect from weight decay. This would neatly explain why grokking takes longer as weight decay decreases~\citep{power2021grokking}, and does not completely vanish in the absence of weight decay. Given that there is a potential extension of our theory that explains grokking without weight decay, and the significant confirming evidence that we have found for our theory in settings with weight decay, we are overall confident that our explanation is at least one part of the true explanation when weight decay is present.

\paragraph{Broader applicability: beyond parameter norm.} Another limitation is that we only consider one kind of constraint that gradient descent must navigate: parameter norm. Typically, there are many other constraints: fitting the training data, capacity in ``bottleneck activations''~\citep{elhage2021mathematical}, interference between circuits~\citep{elhage2022superposition}, and more. This may limit the broader applicability of our theory, despite its success in explaining grokking.

\paragraph{Broader applicability: realistic settings.} A natural question is what implications our theory has for more realistic settings. We expect that the general concepts of circuits, efficiency, and speed of learning continue to apply. However, in realistic settings, good training performance is typically achieved when the model has many different circuit families that contribute different aspects (e.g. language modelling requires spelling, grammar, arithmetic, etc). We expect that these will have a wide variety of learning speeds and efficiencies (although note that ``efficiency'' is not as well defined in this setting, because the circuits don't get perfect training accuracy).

In contrast, the key property for grokking in ``algorithmic'' tasks like modular arithmetic is that there are two clusters of circuit families -- one slowly learned, efficient, generalising cluster, and one quickly learned, inefficient, memorising cluster. In particular, our explanation relies on there being no circuits in between the two clusters. Therefore we observe a sharp transition in test performance when shifting from the memorising to the generalising cluster.

\paragraph{Future work.} Within grokking, several interesting puzzles are still left unexplained. Why does the time taken to grok rise super-exponentially as dataset size decreases? How does the random initialisation interact with efficiency to determine which circuits are found by gradient descent? What causes generalising circuits to develop slower? Investigating these puzzles is a promising avenue for further work.

While the direct application of our work is to understand the puzzle of grokking, we are excited about the potential for understanding deep learning more broadly through the lens of circuit efficiency. We would be excited to see work looking at the role of circuit efficiency in more realistic settings, and work that extends circuit efficiency to consider other constraints that gradient descent must navigate.

\section{Conclusion}

The central question of our paper is: in grokking, why does the network's test performance improve dramatically upon continued training, having already achieved nearly perfect training performance? Our explanation is: the generalising solution is more ``efficient'' but slower to learn than the memorising solution. After quickly learning the memorising circuit, gradient descent can still decrease loss even further by simultaneously strengthening the efficient, generalising circuit and weakening the inefficient, memorising circuit. 

Based on our theory we predict and demonstrate two novel behaviours: \emph{ungrokking}, in which a model that has perfect generalisation returns to memorisation when it is further trained on a dataset with size smaller than the critical threshold, and \emph{semi-grokking}, where we train a randomly initialised network on the critical dataset size which results in a grokking-like transition to middling test accuracy. Our explanation is the only one we are aware of that has made (and confirmed) such surprising advance predictions, and we have significant confidence in the explanation as a result.

\cameraOnly{
\section*{Acknowledgements}

Thanks to Paul Christiano, Xander Davies, Seb Farquhar, Geoffrey Irving, Tom Lieberum, Eric Michaud, Vlad Mikulik, Neel Nanda, Jonathan Uesato, and anonymous reviewers for valuable discussions and feedback.
}

\bibliography{references}

\newpage

\appendix

\section{Experimental details and more evidence}
\label{app:experimental-details}

For all our experiments, we use 1-layer decoder-only transformer networks~\citep{vaswani2017transformers} with learned positional embeddings, untied embeddings/unembeddings, The hyperparameters are as follows: $d_\text{model} = 128$ is the residual stream width, $d_\text{head} = 32$ is the size of the query, key, and value vectors for each attention head, $d_\text{mlp} = 512$ is the number of neurons in the hidden layer of the MLP, and we have $d_\text{model}/d_\text{head} = 4$ heads per self-attention layer. We optimise the network with full batch training (that is, using the entire training dataset for each update) using the AdamW optimiser~\citep{loshchilov2019adamw} with $\beta_1=0.9$, $\beta_2=0.98$, learning rate of $10^{-3}$, and weight decay of $1.0$. In some of our experiments we vary the weight decay in order to produce networks with varying parameter norm.

Following \cite{power2021grokking}, for a binary operation $x \circ y$, we construct a dataset of the form $\token{x}\token{\circ}\token{y}\token{=}\token{x \circ y}$, where $\token{a}$ stands for the token corresponding to the element $a$. We choose a fraction of this dataset at random as the train dataset, and the remainder as the test dataset. The first 4 tokens $\token{x}\token{\circ}\token{y}\token{=}$ are the input to the network, and we train with cross-entropy loss over the final token $\token{x \circ y}$. For all modular arithmetic tasks we use the modulus $p = 113$, so for example the size of the full dataset for modular addition is $p^2 = 12769$, and $d_\text{vocab} = 115$, including the $\token{+}$ and $\token{=}$ tokens.

\subsection{Semi-grokking}

\begin{figure}
    \centering
    \includegraphics[width=\linewidth]{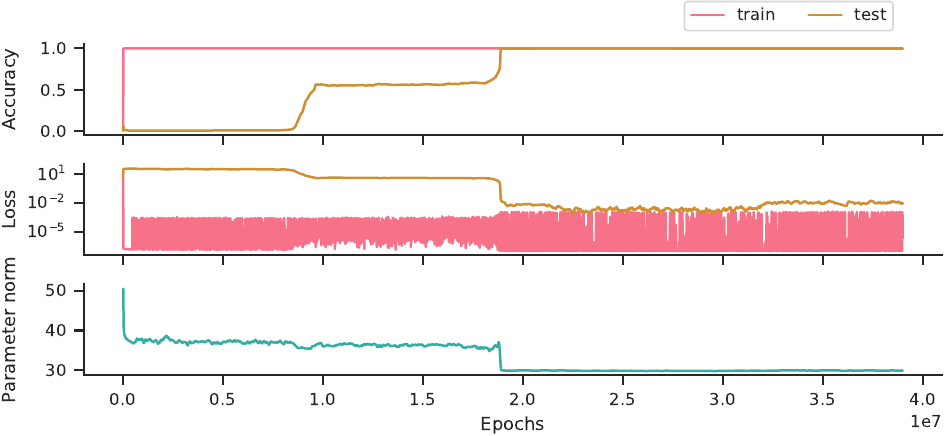}
    \caption{\textbf{Examining a single semi-grokking run in detail.} We plot accuracy, loss, and parameter norm over training for a single cherry-picked modular addition run at a dataset size of 1532 (12\% of the full dataset). This run shows transient semi-grokking. At epoch $0.8 \times 10^7$, test accuracy rises to around $0.55$, and then stays there for $10^7$ epochs, because \Gen and \Mem efficiencies are balanced. At epoch $1.8 \times 10^7$, we speculate that gradient descent finds an even more efficient \Gen circuit, as parameter norm drops suddenly and test accuracy rises to 1. At epoch $3.2 \times 10^7$ we see test loss \textit{rise}, we do not know why. There seem to be multiple phases, perhaps corresponding to the network transitioning between mixtures of multiple circuits with increasing efficiencies, but further investigation is needed.}
    \label{fig:semi_grokking_story}
\end{figure}

In Section~\ref{sec:semigrokking} we looked at all semi-grokking training runs in Figure~\ref{fig:semi_grokking_more_seeds}. Here, we investigate a single example of transient semi-grokking in more detail (see Figure~\ref{fig:semi_grokking_story}). We speculate that there are multiple circuits with increasing efficiencies for \Gen, and in these cases the more efficient circuits are slower to learn. This would explain transient semi-grokking: gradient descent first finds a less efficient \Gen and we see middling generalisation, but since we are using the upper range of $\criticalDatasetSize$, eventually gradient descent finds a more efficient \Gen leading to full generalisation.

\subsection{Ungrokking}

In Figure~\ref{fig:ungrokking_more_seeds}, we show many ungrokking runs for modular addition, and in Figure~\ref{fig:ungrokking_many_tasks} we show ungrokking across many other tasks.

We have already seen that $\criticalDatasetSize$ is affected by the random initialisation. It is interesting to compare $\criticalDatasetSize$ when starting with a given random initialisation, and when ungrokking from a network that was trained to full generalisation with the same random initialisation. Figure~\ref{fig:semi_grokking_more_seeds} shows a semi-grokking run that achieves a test accuracy of $\sim$0.7 with a dataset size of $\sim$2000, while Figure~\ref{fig:ungrokking_more_seeds} shows ungrokking runs that achieve a test accuracy of $\sim$0.7 with a dataset size of around 800--1000, less than half of what the semi-grokking run required.

In Figure~\ref{fig:ungrokking_accuracy_vs_dataset_size}, the final test accuracy after \emph{ungrokking} shows a smooth relationship with dataset size, which we might expect if \Gen is getting stronger on a smoothly increasing number of inputs compared to \Mem. However due to the difficulties discussed previously, we don't see a smooth relationship between test accuracy and dataset size in \emph{semi-grokking}.

These results suggest that $\criticalDatasetSize$ is an oversimplified concept, because in reality the initialisation and training dynamics affect which circuits are found, and therefore the dataset size at which we see middling generalisation.

\begin{figure}
    \centering
    \includegraphics[width=\linewidth]{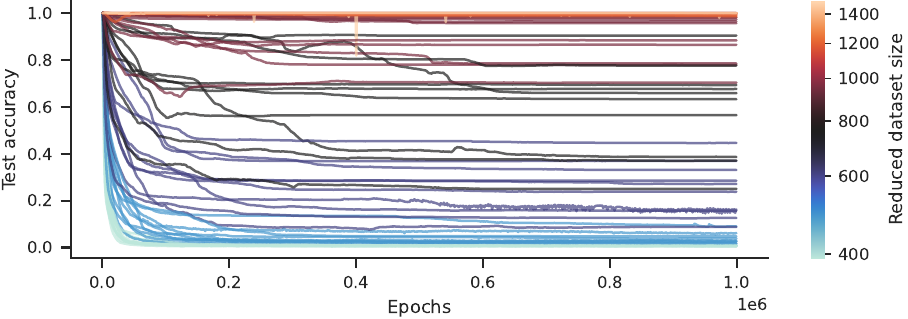}
    \caption{\textbf{Many ungrokking runs.} We show test accuracy over epochs for a range of ungrokking runs for modular addition. Each line represents a single run, and we sweep over 7 geometrically spaced dataset sizes in $[390, 1494]$ with 10 seeds each. Each run is initialised with parameters from a network trained on the full dataset (the initialisation runs are not shown), so test accuracy starts at 1 for all runs. When the dataset size is small enough, the network ungroks to poor test accuracy, while train accuracy remains at 1 (not shown). For an intermediate dataset size, we see ungrokking to middling test accuracy as \Gen and \Mem efficiencies are balanced.}
    \label{fig:ungrokking_more_seeds}
\end{figure}

\begin{figure}
    \centering
    \includegraphics[width=\linewidth]{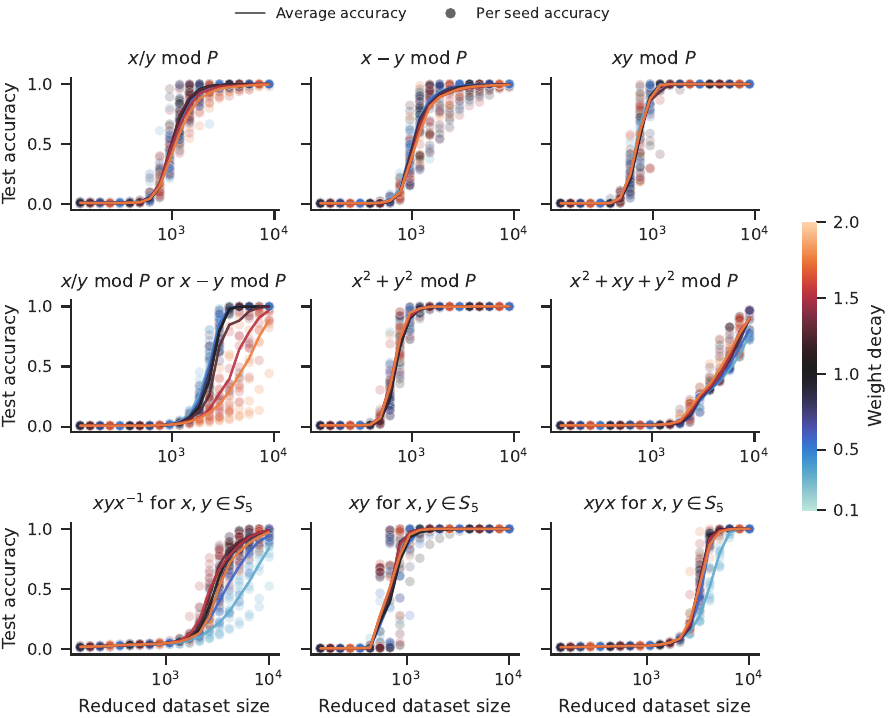}
    \caption{\textbf{Ungrokking on many other tasks.} We plot test accuracy against reduced dataset size for many other modular arithmetic and symmetric group tasks~\citep{power2021grokking}. For each run, we train on the full dataset (achieving 100\% accuracy), and then further train on a reduced subset of the dataset for 100k steps. The results show clear ungrokking, since in many cases test accuracy falls below 100\%, often to nearly 0\%. For most datasets the transition point is independent of weight decay (different coloured lines almost perfectly overlap).}
    \label{fig:ungrokking_many_tasks}
\end{figure}

\subsection{\Gen and \Mem development during grokking}

In Figure~\ref{fig:grokking_story} we show \Gen and \Mem development via the proxy measures defined in Appendix~\ref{app:trig-explanation} for a randomly-picked grokking run. Looking at these measures was very useful to form a working theory for why grokking happens. However as we note in Appendix~\ref{app:trig-explanation}, these proxy measures tend to overestimate \Gen and underestimate \Mem.

We note some interesting phenomena in Figure~\ref{fig:grokking_story}:
\begin{enumerate}
    \item Between epochs 200 to 1500, \emph{both} the \Gen and \Mem logits are rising while parameter norm is falling, indicating that gradient descent is improving efficiency (possibly by removing irrelevant parameters).
    \item After epoch 4000, the \Gen logit \emph{falls} while the \Mem logit is already $\sim$0. Since test loss continues to fall, we expect that incorrect logits from \Mem on the test dataset are getting cleaned up, as described in \citet{nanda2023grokking}.
\end{enumerate}

\begin{figure}
    \centering
    \includegraphics[width=\linewidth]{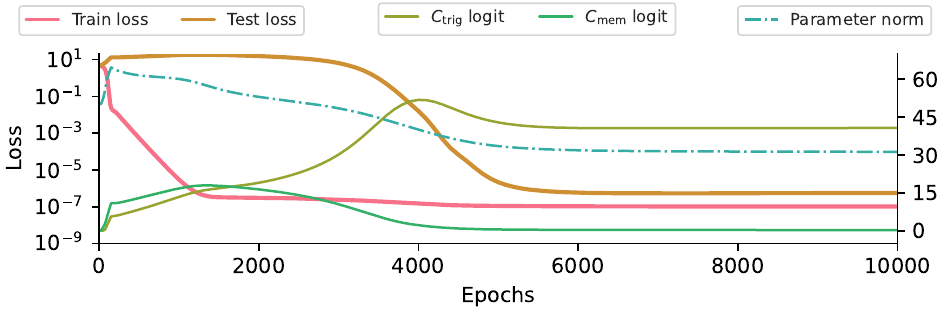}
    \caption{\textbf{Grokking occurs because \Gen is more efficient than \Mem.} We show loss, parameter norm, and the value of the correct logit for \Gen and \Mem for a randomly-picked training run. By step 200, the train accuracy is already perfect (not shown), train loss is low while test loss has risen, and parameter norm is at its maximum value, indicating strong \Mem. Train loss continues to fall rapidly until step 1500, as parameter norm falls and the \Gen logit becomes higher than the \Mem logit. At step 3500, test loss starts to fall as the high \Gen logit starts to dominate, and by step 6000 we get good generalisation.}
    \label{fig:grokking_story}
\end{figure}

\begin{figure}[t]
    \centering
    \begin{subfigure}[t]{0.43\textwidth}
        \centering
        \includegraphics[width=\linewidth]{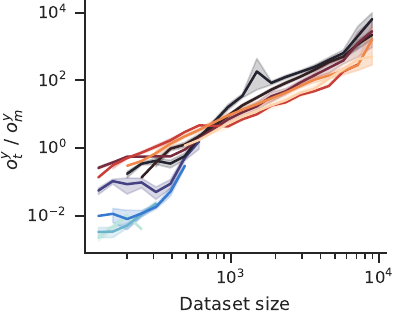}
        \caption{\textbf{Logit ratio ($\logit{\Label{}}{t}/\logit{\Label{}}{m})$ vs dataset size ($\datasetSize$).} Colours correspond to different bucketed values of parameter norm ($\paramNorm{}$). Each line shows that as dataset size increases, a fixed parameter norm (fixed colour) is being reallocated smoothly towards increasing the trigonometric logit compared to the memorisation logit.}
        \label{fig:ttm_ratio_vs_dataset_size}
    \end{subfigure} \hfill
    \begin{subfigure}[t]{0.39\textwidth}
        \centering
        \includegraphics[width=\linewidth]{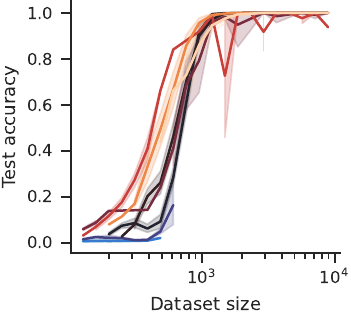}
        \caption{\textbf{Test accuracy vs dataset size ($\datasetSize$).} We see a smooth dependence on dataset size. Each line shows that as dataset size increases, the reallocation of a fixed parameter norm (fixed colour) towards \Gen from \Mem results in increasing accuracy.}
        \label{fig:ungrokking_accuracy_vs_dataset_size}
    \end{subfigure} \hfill
    \begin{subfigure}[t]{0.14\textwidth}
        \centering
        \vspace{-5.5cm}
        \includegraphics[width=\linewidth,trim={1.5cm 0 0 0},clip]{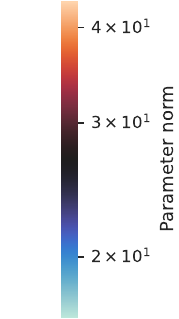}
    \end{subfigure}
    \caption{\textbf{Relative strength at convergence.} We report logit ratios and test accuracy at convergence across a range of training runs, generated by sweeping over the weight decay and random seed to obtain different parameter norms at the same dataset size. We use the ungrokking runs from Figure~\ref{fig:ungrokking_more_seeds}, so every run is initialised with parameters obtained by training on the full dataset.}
    \label{fig:trig_to_mem}
\end{figure}

\subsection{Tradeoffs between \Gen and \Mem} \label{sec:both-circuits-empirical-evidence}

In Section~\ref{sec:efficiency-and-dataset-size} we looked at the efficiency of \Gen-only and \Mem-only circuits. In this section we train on varying dataset sizes so that the network develops a mixture of \Gen and \Mem circuits, and study their relative strength using the correct logit as a proxy measure (described in Appendix~\ref{app:trig-explanation}).

As we demonstrated previously, \Mem's efficiency drops with increasing dataset size, while \Gen's stays constant. Theorem~\ref{thm:efficiency-circuit-weights-logits} (case 2) suggests that parameter norm allocated to a circuit is proportional to efficiency, and since logit values also increase with parameter norm, this implies that the ratio of the \Gen to \Mem logit $\logit{\Label{}}{t}/\logit{\Label{}}{m}$ should increase monotonically with dataset size.

In Figure~\ref{fig:ttm_ratio_vs_dataset_size} we see exactly this: the logit ratio changes monotonically (over 6 orders of magnitude) with increasing dataset size.

Due to the difficulties in training to convergence at small dataset sizes, we initialised all parameters from a \Gen-only network trained on the full dataset. We confirmed that in all the runs, at convergence, the training loss from the \Gen-initialised network was lower than the training loss from a randomly initialised network, indicating that this initialisation allows our optimiser to find a better minimum than from random initialisation.

\section{\Gen and \Mem in modular addition} \label{app:trig-explanation}
    In modular addition, given two integers $a, b$ and a modulus $\modulus$ as input, where $0 \le a,b < \modulus$, the task is to predict $a + b \mod \modulus$. \citet{nanda2023grokking} identified the generalising algorithm implemented by a 1-layer transformer after grokking (visualised in Figure~\ref{fig:trig_explanation}), which we call the ``trigonometric'' algorithm. In this section we summarise the algorithm, and explain how we produce our proxy metrics for the strength of \Gen and \Mem.

    \paragraph{Trigonometric logits.} We explain the structure of the logits produced by the trigonometric algorithm. For each possible label $\ModAddLabel \in \{0, 1, \dots \modulus-1\}$, the trigonometric logit $\logit{\ModAddLabel}{}$ will be given by $\sum_{\omega_k} \cos(\omega_k(a + b - \ModAddLabel))$, for a few key frequencies $\omega_k = 2\pi \frac{k}{\modulus}$ with integer $k$. For the true label $\ModAddLabelTrue = a + b \mod \modulus$, the term $\omega_k(a + b - \ModAddLabelTrue)$ is an integer multiple of $2\pi$, and so $\cos(\omega_k(a + b - \ModAddLabelTrue)) = 1$. For any incorrect label $\ModAddLabel \neq a + b \mod \modulus$, it is very likely that at least \emph{some} of the key frequencies satisfy $\cos(\omega_k(a + b - \ModAddLabel)) \ll 1$, creating a large difference between $\logit{\ModAddLabel}{}$ and $\logit{\ModAddLabelTrue}{}$.

    \paragraph{Trigonometric algorithm.}
    \begin{figure}[t]
        \centering
        \includegraphics[width=\textwidth]{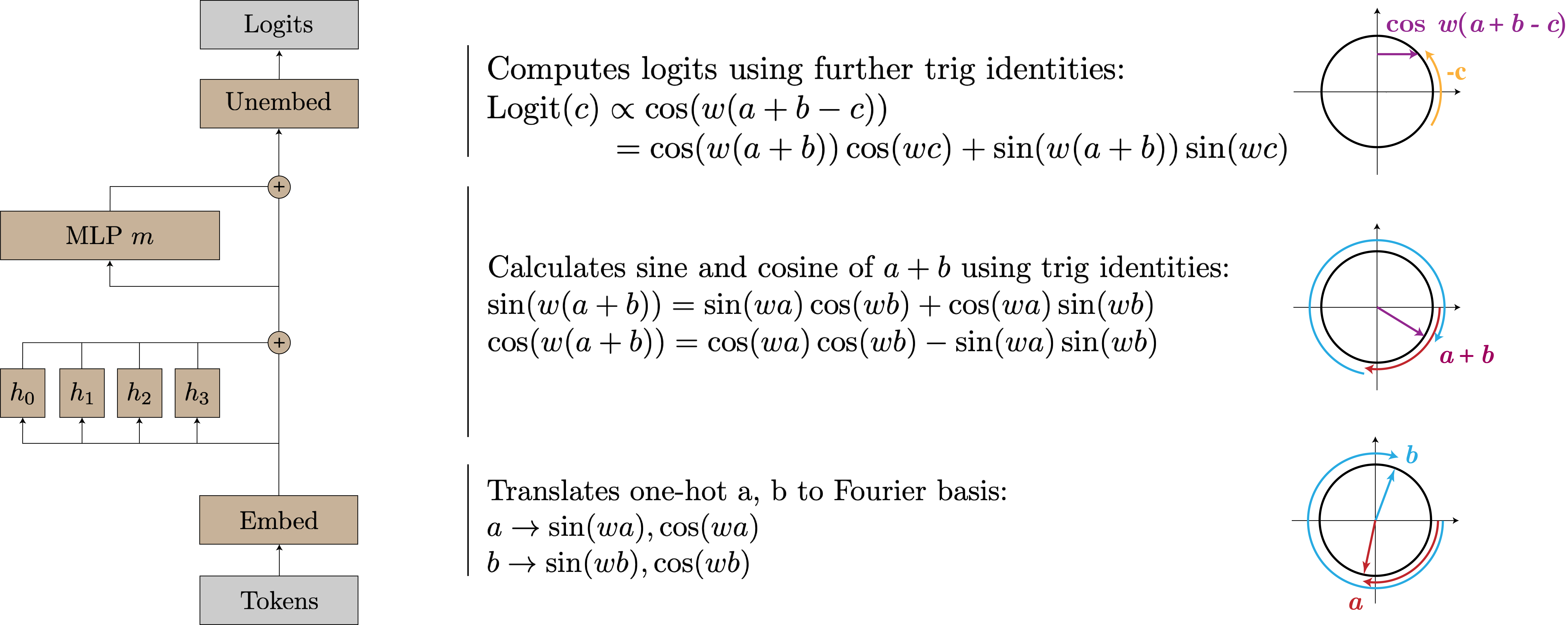}
        \caption{\textbf{The trigonometric algorithm for modular arithmetic} (reproduced from~\citet{nanda2023grokking}). Given two numbers $a$ and $b$, the model projects each point onto a corresponding rotation using its embedding matrix. Using its attention and MLP layers, it then composes the rotations to get a representation of $a + b \mod \modulus$. Finally, it ``reads off'' the logits for each $c \in \{0, 1, \dots, \modulus - 1\}$, by rotating by $-c$ to get $cos(\omega(a + b - c))$, which is maximised when $a + b \equiv c \mod P$ (since $\omega$ is a multiple of $2\pi$).}
        \label{fig:trig_explanation}
    \end{figure}

    There is a set of key frequencies $\omega_k$. (These frequencies are typically whichever frequencies were highest at the time of random initialisation.) For an arbitrary label $\ModAddLabel$, the logit $\logit{\ModAddLabel}{}$ is computed as follows:
    \begin{enumerate}
        \item Embed the one-hot encoded number $a$ to $\sin(\omega_{k}a)$ and $\cos(\omega_{k}a)$ for the various frequencies $\omega_k$. Do the same for $b$.
        \item Compute $\cos(\omega_{k}(a + b))$ and $\sin(\omega_{k}(a + b))$ using the intermediate attention and MLP layers via the trigonometric identities:
        \begin{align*}
            \cos(\omega_{k}(a + b)) = \cos(\omega_{k}a) \cos(\omega_{k}a) - \sin(\omega_{k}a) \sin(\omega_{k}b)\\
            \sin(\omega_{k}(a + b)) = \sin(\omega_{k}a) \cos(\omega_{k}b) + \cos(\omega_{k}a) \sin(\omega_{k}b)
        \end{align*}
        \item Use the output and unembedding matrices to implement the trigonometric identity:
        \begin{align*}
            \logit{\ModAddLabel}{} = \sum_{\omega_k} \cos(\omega_{k}(a + b - \ModAddLabel)) = \sum_{\omega_k} \cos(\omega_{k}(a + b)) \cos(\omega_{k}\ModAddLabel) + \sin(\omega_{k}(a + b)) \sin(\omega_{k}\ModAddLabel).
        \end{align*}
    \end{enumerate} 

    \paragraph{Isolating trigonometric logits.} Given a classifier $\classifier$, we can aggregate its logits on every possible input, resulting in a vector $\vec{Z}_{\classifier}$ of length $\modulus^3$ where $\vec{Z}_{\classifier}^{a,b,\ModAddLabel} = \logit{\ModAddLabel}{\classifier}(\text{``}a\texttt{ + }b\texttt{ =}\text{''})$ is the logit for label $\ModAddLabel$ on the input $(a, b)$. We are interested in identifying the contribution of the trigonometric algorithm to $\vec{Z}_{\classifier}$. We use the same method as \cite{chughtai2023grokking} and restrict $\vec{Z}_{\classifier}$ to a much smaller trigonometric subspace.

    For a frequency $\omega_k$, let us define the $\modulus^3$-dimensional vector $\vec{Z}_{\omega_k}$ as $\vec{Z}^{a,b,\ModAddLabel}_{\omega_k} = \cos(\omega_{k}(a + b - \ModAddLabel))$. Since $\vec{Z}_{\omega_k} = \vec{Z}_{\omega_{\modulus - k}}$, we set $1 \le k \le K$, where $K = \lceil(\modulus - 1) / 2\rceil$, to obtain $K$ distinct vectors, ignoring the constant bias vector. These vectors are orthogonal, as they are part of a Fourier basis.
    
    Notice that any circuit that was exactly following the learned algorithm described above would only produce logits in the directions $\vec{Z}_{\omega_k}$ for the key frequencies $\omega_k$. So, we can define the trigonometric contribution to $\vec{Z}_{\classifier}$ as the projection of $\vec{Z}_{\classifier}$ onto the directions $\vec{Z}_{\omega_k}$. We may not know the key frequencies in advance, but we can sum over all $K$ of them, giving the following definition for trigonometric logits:
    \begin{align*}
      \vec{Z}_{\classifier, T} = \sum_{k=1}^{K} (\vec{Z}_{\classifier} \cdot \hat{Z}_{\omega_k}) \hat{Z}_{\omega_k}
    \end{align*}
    where $\hat{Z}_{\omega_k}$ is the normalised version of $\vec{Z}_{\omega_k}$. This corresponds to projecting onto a $K$-dimensional subspace of the $\modulus^3$-dimensional space in which $\vec{Z}_{\classifier}$ lives.

    \paragraph{Memorisation logits.} Early in training, neural networks memorise the training dataset without generalising, suggesting that there exists a memorisation algorithm, implemented by the circuit \Mem\footnote{In reality, there are at least two different memorisation algorithms: commutative memorisation (which predicts the same answer for $(a, b)$ and $(b, a)$) and non-commutative memorisation (which does not). However, this difference does not matter for our analyses, and we will call both of these ``memorisation'' in this paper.}. Unfortunately, we do not understand the algorithm underlying memorisation, and so cannot design a similar procedure to isolate \Mem's contribution to the logits. However, we hypothesise that for modular addition, \Gen and \Mem are the only two circuit families of importance for the loss. This allows us to define the \Mem contribution to the logits as the residual:
    \begin{align*}
      \vec{Z}_{\classifier, M} = \vec{Z}_{\classifier} - \vec{Z}_{\classifier, T}    
    \end{align*}

    \paragraph{\Trig and \Mem circuits.} We say that a circuit is a \Trig circuit if it implements the \Trig algorithm, and similarly for \Mem circuits. Importantly, this is a many-to-one mapping: there are many possible circuits that implement a given algorithm.

    We isolate \Trig ($\logits{t}$) and \Mem ($\logits{m}$) logits by projecting the output logits ($\logits{}$) as described in Appendix~\ref{app:trig-explanation}. We cannot directly measure the circuit weights $\weight{t}$ and $\weight{m}$, but instead use an indirect measure: the value of the logit for the correct class given by each circuit, i.e. $\logit{\Label{}}{t}$ and $\logit{\Label{}}{m}$.
    
    \paragraph{Flaws} These metrics should be viewed as an imperfect proxy measure for the true strength of the trigonometric and memorisation circuits, as they have a number of flaws:
    \begin{enumerate}
        \item When both \Trig and \Mem are present in the network, they are both expected to produce high values for the correct logits, and low values for incorrect logits, on the train dataset. Since the \Trig and \Mem logits are correlated, it becomes more likely that $\vec{Z}_{\classifier, T}$ captures \Mem logits too.
        \item In this case we would expect our proxy measure to overestimate the strength of \Trig and underestimate the strength of \Mem. In fact, in our experiments we do see large negative correct logit values for \Mem on training for semi-grokking, which probably arises because of this effect.
        \item Logits are not inherently meaningful; what matters for loss is the extent to which the correct logit is larger than the incorrect logits. This is not captured by our proxy metric, which only looks at the size of the correct logit. In a binary classification setting, we could instead use the difference between the correct and incorrect logit, but it is not clear what a better metric would be in the multiclass setting.
    \end{enumerate}

\section{Details for the minimal example}
\label{app:simulation-details}

In Figure~\ref{fig:sim-result} we show that two ingredients: multiple circuits with different efficiencies, and slow and fast circuit development, are sufficient to reproduce learning curves that qualitatively demonstrate grokking. In Table~\ref{tab:hypers-simulation} we provide details about the simulation used to produce this figure.

As explained in Section~\ref{sec:theory}, the logits produced by \Gen and \Mem are given by:

\begin{align}
    \logit{\Label{}}{\gen}(\Input{}) &= \indicator{(\Input{}, \Label{}) \in \Dataset\text{ or }(\Input{}, \Label{}) \in \TestDataset} \label{eq:sim-gen-logits}\\
    \logit{\Label{}}{\mem}(\Input{}) &= \indicator{(\Input{}, \Label{}) \in \Dataset\text{ or }(\Input{}, \Label{}) \in \MemDataset} \label{eq:sim-mem-logits}
\end{align}

These are scaled by two independent weights for each circuit, giving the overall logits as:
\begin{equation}\label{eq:sim-overall-logits}
    \logit{\Label{}}{}(\Input{}) = \weight{\gen_1}\weight{\gen_2}\logit{\Label{}}{\gen}(\Input{}) + \weight{\mem_1}\weight{\mem_2}\logit{\Label{}}{\mem}(\Input{})
\end{equation}

We model the parameter norms according to the scaling efficiency in Section~\ref{sec:efficiency-minimal-example}, inspired by a $\scalingExp$-layer MLP with Relu activations and without biases:
\begin{align*}
    \circuitNorm{c} = (\weight{c_1}\weight{c_2})^{1/\scalingExp}\paramNorm{c}~\text{for}~c \in (g, m).
\end{align*}

From \Crefrange{eq:sim-gen-logits}{eq:sim-overall-logits} we get the following equations for train and test loss respectively:
\begin{align*}
    \Loss_{\textrm{train}} &= - \log \frac{\exp(\weight{g_1} \weight{g_2} + \weight{m_1} \weight{m_2})}{(q-1) +\exp(\weight{g_1} \weight{g_2} + \weight{m_1} \weight{m_2})} + \LossWD,
    \\
    \Loss_{\textrm{test}} &= - \log \frac{\exp(\weight{g_1} \weight{g_2})}{(q-2) +\exp(\weight{g_1} \weight{g_2}) + \exp(\weight{m_1} \weight{m_2})} + \LossWD,
\end{align*}
where $q$ is the number of labels, and the weight decay loss is:
\begin{align*}
    \LossWD &= \circuitNorm{g}^2 + \circuitNorm{m}^2.
\end{align*}

The weights $\weight{c_i}$ are updated based on gradient descent:
\begin{align*} \label{eq:dynamics}
    \weight{c_i}(\tau)  \gets \weight{c_i}(\tau - 1) - \learningRate \frac{\partial \Loss_{\textrm{train}}}{\partial \weight{c_i}}
\end{align*}
where $\learningRate$ is a learning rate. The initial values of the parameters are $\weight{c_i}(0)$. In Table~\ref{tab:hypers-simulation} we list the values of the simulation hyperparameters.

\begin{table}[t]
\caption{Hyperparameters used for our simulations.}
\label{tab:hypers-simulation}
\begin{subtable}[t]{0.3\linewidth}
\centering
\caption{\Gen learned slower but more efficient than \Mem.}
\begin{tabular}{@{}cc@{}}
\toprule
\textbf{Parameter} & \textbf{Value}\\\midrule

$\paramNorm{g}$ & $1$ \\ \hline

$\paramNorm{m}$ & $2$\\ \hline

$\scalingExp$ & $1.2$ \\ \hline
  
$\weightDecayHyper$ & $0.005$ \\ \hline

$\weight{g_1}(0)$ & $0$ \\ \hline

$\weight{g_2}(0)$ & $0.005$ \\ \hline
  
$\weight{m_1}(0)$ & $0$ \\ \hline

$\weight{m_2}(0)$ & $1$ \\ \hline

$q$ & $113$ \\ \hline

$\learningRate$ & $0.01$ \\\bottomrule

\end{tabular}
\end{subtable}
\hfill
\begin{subtable}[t]{0.3\linewidth}
\centering
\caption{\Gen less efficient than \Mem.\newline}
\begin{tabular}{@{}cc@{}}
\toprule
\textbf{Parameter} & \textbf{Value}\\\midrule

$\paramNorm{g}$ & $4$ \\ \hline

$\paramNorm{m}$ & $2$\\ \hline

$\scalingExp$ & $1.2$ \\ \hline
  
$\weightDecayHyper$ & $0.005$ \\ \hline

$\weight{g_1}(0)$ & $0$ \\ \hline

$\weight{g_2}(0)$ & $0.005$ \\ \hline
  
$\weight{m_1}(0)$ & $0$ \\ \hline

$\weight{m_2}(0)$ & $1$ \\ \hline

$q$ & $113$ \\ \hline

$\learningRate$ & $0.01$ \\\bottomrule

\end{tabular}
\end{subtable}
\hfill
\begin{subtable}[t]{0.3\linewidth}
\centering
\caption{\Gen and \Mem learned at equal speeds.}
\begin{tabular}{@{}cc@{}}
\toprule
\textbf{Parameter} & \textbf{Value}\\\midrule

$\paramNorm{g}$ & $1$ \\ \hline

$\paramNorm{m}$ & $2$\\ \hline

$\scalingExp$ & $1.2$ \\ \hline
  
$\weightDecayHyper$ & $0.005$ \\ \hline

$\weight{g_1}(0)$ & $0$ \\ \hline

$\weight{g_2}(0)$ & $1$ \\ \hline

$\weight{m_1}(0)$ & $0$ \\ \hline

$\weight{m_2}(0)$ & $1$ \\ \hline

$q$ & $113$ \\ \hline

$\learningRate$ & $0.01$ \\\bottomrule

\end{tabular}
\end{subtable}
\end{table}

\section{Proofs of theorems} \label{app:proofs}

We assume we have a set of inputs $\Inputs$, a set of labels $\Labels$, and a training dataset, $\Dataset = \set{ (\Input{1}, \Label{1}), \dots (\Input{D}, \Label{D})}$. Let $\classifier$ be a classifier that assigns a real-valued logit for each possible label given an input. We denote an individual logit as $\logit{\Label{}}{\classifier}(\Input{}) \coloneqq \classifier(\Input{}, \Label{})$. When the input $\Input{}$ is clear from context, we will denote the logit as $\logit{\Label{}}{\classifier}$. Excluding weight decay, the \emph{loss} for the classifier is given by the softmax cross-entropy loss:
\begin{align*}
    \LossXEnt(\classifier) = -\frac{1}{\numTrainExamples} \sum\limits_{(\Input{}, \Label{}) \in \Dataset} \log \frac{\exp(\logit{\Label{}}{\classifier})}{\sum\limits_{\LabelVar \in \Labels} \exp(\logit{\LabelVar}{\classifier})}.
\end{align*}

For any $c \in \Reals$, let $c \cdot \classifier$ be the classifier whose logits are multipled by $c$, that is, $(c \cdot \classifier)(\Input{}, \Label{}) = c \times \classifier(\Input{}, \Label{})$. Intuitively, once a classifier achieves perfect accuracy, then the true class logit $\logit{\LabelTrue{}}{}$ will be larger than any incorrect class logit $\logit{\LabelVar{}}{}$, and so loss can be further reduced by scaling up \emph{all} of the logits further (increasing the gap between $\logit{\LabelTrue{}}{}$ and $\logit{\LabelVar{}}{}$).

\begin{theorem} \label{thm:scaling-logits-decreases-loss}
Suppose that the classifier $\classifier$ has perfect accuracy, that is, for any $(\Input{}, \LabelTrue{}) \in \Dataset$ and any $\LabelVar \neq \LabelTrue{}$ we have $\logit{\LabelTrue{}}{\classifier} > \logit{\LabelVar}{\classifier}$. Then, for any $c > 1$, we have $\LossXEnt(c \cdot \classifier) < \LossXEnt(\classifier)$.
\end{theorem}
\begin{proof}
First, note that we can rewrite the loss function as:
\begin{equation*}
    \LossXEnt(\classifier) = -\frac{1}{\numTrainExamples} \sum\limits_{(\Input{}, \LabelTrue{})} \log \frac{\exp(\logit{\LabelTrue{}}{\classifier})}{\sum\limits_{\LabelVar} \exp(\logit{\LabelVar}{\classifier})}
    = \frac{1}{\numTrainExamples} \sum\limits_{(\Input{}, \LabelTrue{})} \log \left(\frac{\sum\limits_{\LabelVar} \exp(\logit{\LabelVar}{\classifier})}{\exp(\logit{\LabelTrue{}}{\classifier})} \right)
    = \frac{1}{\numTrainExamples} \sum\limits_{(\Input{}, \LabelTrue{})} \log \left(1 + \sum\limits_{\LabelVar \neq \LabelTrue{}} \exp(\logit{\LabelVar}{\classifier} - \logit{\LabelTrue{}}{\classifier}) \right)
\end{equation*}

Since we are given that $\logit{\LabelTrue{}}{\classifier} > \logit{\LabelVar}{\classifier}$, for any $c > 1$ we have $c(\logit{\LabelVar}{\classifier} - \logit{\LabelTrue{}}{\classifier})) < \logit{\LabelVar}{\classifier} - \logit{\LabelTrue{}}{\classifier}$. Since $\exp$, $\log$, and sums are all monotonic, this gives us our desired result:

\begin{equation*}
    \LossXEnt(c \cdot \classifier) = \frac{1}{\numTrainExamples} \sum\limits_{(\Input{}, \LabelTrue{})} \log \left(1 + \sum\limits_{\LabelVar \neq \LabelTrue{}} \exp(c(\logit{\LabelVar}{\classifier} - \logit{\LabelTrue{}}{\classifier})) \right) < \frac{1}{\numTrainExamples} \sum\limits_{(\Input{}, \LabelTrue{})} \log \left(1 + \sum\limits_{\LabelVar \neq \LabelTrue{}} \exp(\logit{\LabelVar}{\classifier} - \logit{\LabelTrue{}}{\classifier}) \right) = \LossXEnt(\classifier).
\end{equation*}
\end{proof}

We now move on to Theorem~\ref{thm:efficiency-circuit-weights-logits}. First we establish some basic lemmas that will be used in the proof:
\begin{lemma} \label{lemma:sum-of-powers}
Let $a, b, r \in \Reals$ with $a, b \geq 0$ and $0 < r \leq 1$. Then $(a + b)^r \leq a^r + b^r$.
\end{lemma}
\begin{proof}
The case with $a = 0$ or $b = 0$ is clear, so let us consider $a, b > 0$. Let $x = \frac{a}{a+b}$ and $y = \frac{b}{a+b}$. Since $0 \leq x \leq 1$, we have $x^{(1-r)} \leq 1$, which implies $x \leq x^r$. Similarly $y \leq y^r$. Thus $x^r + y^r \geq x + y = 1$. Substituting in the values of $x$ and $y$ we get $\frac{a^r + b^r}{(a + b)^r} \geq 1$, which when rearranged gives us the desired result.
\end{proof}

\begin{lemma} \label{lemma:power-derivative}
For any $x, c, r \in \Reals$ with $r \geq 1$, there exists some $\delta > 0$ such that for any $\epsilon < \delta$ we have $x^r - (x - \epsilon)^r > \delta (r x^{r-1} - c)$.
\end{lemma}
\begin{proof}
The function $f(x) = x^r$ is everywhere-differentiable and has derivative $rx^{r-1}$. Thus we can choose $\delta$ such that for any $\epsilon < \delta$ we have $-c < \frac{x^r - (x - \epsilon)^r}{\delta} - rx^{r-1} < c$. Rearranging, we get $x^r - (x - \epsilon)^r > \delta (r x^{r-1} - c)$ as desired.
\end{proof}

\subsection{Weight decay favours efficient circuits} \label{sec:efficiency-minimal-example}

To flesh out the argument in Section~\ref{sec:theory}, we construct a minimal example of multiple circuits $\{\circuit{1}, \dots \circuit{\numCircuits}\}$ of varying efficiencies that can be scaled up or down through a set of non-negative \emph{weights} $\weight{i}$. Our classifier is given by $\classifier = \sum_{i=1}^\numCircuits \weight{i}\circuit{i}$, that is, the output $\classifier(\Input{}, \Label{})$ is given by $\sum_{i=1}^\numCircuits \weight{i}\circuit{i}(\Input{}, \Label{})$.

We take circuits $\circuit{i}$ that are \emph{normalised}, that is, they produce the same average logit value. $\paramNorm{i}$ denotes the parameter norm of the normalised circuit $\circuit{i}$. We decide to call a circuit with lower $\paramNorm{i}$ more efficient. However, it is hard to define efficiency precisely. Consider instead the parameter norm $\paramNorm{i}'$ of the scaled circuit $\weight{i}\circuit{i}$. If we define efficiency as either the ratio $\lVert \logits{\circuit{i}} \rVert / \paramNorm{i}'$ or the derivative $d\lVert \logits{\circuit{i}} \rVert/d\paramNorm{i}'$, then it would vary with $\weight{i}$ since $\logits{\circuit{i}}$ and $\paramNorm{i}'$ can in general have different relationships with $\weight{i}$. We prefer $\paramNorm{i}$ as a measure of relative efficiency as it is intrinsic to $\circuit{i}$ rather than depending on its scaling $\weight{i}$.

Gradient descent operates over the weights $\weight{i}$ (but not $\circuit{i}$ or $\paramNorm{i}$) to minimise $\Loss = \LossXEnt + \weightDecayHyper \LossWD$. $\LossXEnt$ can easily be rewritten in terms of $\weight{i}$, but for $\LossWD$ we need to model the parameter norm of the scaled circuits $\weight{i} \circuit{i}$. Notice that, in a $\scalingExp$-layer MLP with Relu activations and without biases, scaling all parameters by a constant $\const$ scales the outputs by $\const^\scalingExp$. Inspired by this observation, we model the parameter norm of $\weight{i} \circuit{i}$ as $\weight{i}^{1/\scalingExp} \paramNorm{i}$ for some $\scalingExp > 0$. This gives the following effective loss:
\begin{equation*} \label{eq:effective-loss-for-weights}
    \Loss(\weightVec) = \LossXEnt\left(\sum\limits_{i=1}^\numCircuits \weight{i}\circuit{i}\right) + \frac{\weightDecayHyper}{2} \sum\limits_{i=1}^\numCircuits (\weight{i}^{\frac{1}{\scalingExp}}\paramNorm{i})^2
\end{equation*}

We will generalise this to any $\lpnorm$-norm (where $\normNum > 0$). Standard weight decay corresponds to $\normNum = 2$. We will also generalise to arbitrary differentiable, bounded training loss functions, instead of cross-entropy loss specifically. In particular, we assume that there is some differentiable $\LossTrain(\classifier)$ such that there exists a finite bound $B \in \mathbb{R}$ such that $\forall \classifier : \LossTrain(\classifier) \geq B$. (In the case of cross-entropy loss, $B = 0$.)

With these generalisations, the overall loss is now given by:
\begin{equation} \label{eq:effective-loss-for-weights-generalised}
    \Loss(\weightVec) = \LossTrain\left(\sum\limits_{i=1}^\numCircuits \weight{i}\circuit{i}\right) + \frac{\weightDecayHyper}{\normNum} \sum\limits_{i=1}^\numCircuits (\weight{i}^{\frac{1}{\scalingExp}}\paramNorm{i})^{\normNum}
\end{equation}

The following theorem establishes that the optimal weight vector allocates more weight to more efficient circuits, under the assumption that the circuits produce identical logits on the training dataset.
\begin{theorem} \label{thm:efficiency-circuit-weights-logits}
Given $\numCircuits$ circuits $\circuit{i}$ and associated $\lpnorm$ parameter norms $\paramNorm{i}$, assume that every circuit produces the same logits on the training dataset, i.e. $\forall i, j$,  $\forall (\Input{}, \_) \in \Dataset$, $\forall \LabelVar{} \in \Labels$ we have $\logit{\LabelVar{}}{\circuit{i}}(\Input{}) = \logit{\LabelVar{}}{\circuit{j}}(\Input{})$. Then, any weight vector $\weightVec^* \in \mathbb{\Reals}^{\numCircuits}$ that minimizes the loss in Equation~\ref{eq:effective-loss-for-weights-generalised} subject to $\weight{i} \geq 0$ satisfies:
\begin{enumerate}
    \item If $\scalingExp \geq \normNum$, then $\weight{i}^* = 0$ for all $i$ such that $\paramNorm{i} > \min_j \paramNorm{j}$.
    \item If $0 < \scalingExp < \normNum$, then $\weight{i}^* \propto \paramNorm{i}^{-\frac{\normNum\scalingExp}{\normNum-\scalingExp}}$.
\end{enumerate}
\end{theorem}
\begin{intuition} Since every circuit produces identical logits, their weights are interchangeable with each other from the perspective of $\LossXEnt$, and so we must analyse how interchanging weights affects $\LossWD$. $\LossWD$ grows as $O(\weight{i}^{2/\scalingExp})$. When $\scalingExp > 2$, $\LossWD$ grows sublinearly, and so it is cheaper to add additional weight to the \emph{largest} weight, creating a ``rich get richer'' effect that results in a single maximally efficient circuit getting all of the weight. When $\scalingExp < 2$, $\LossWD$ grows superlinearly, and so it is cheaper to add additional weight to the \emph{smallest} weight. As a result, every circuit is allocated at least some weight, though more efficient circuits are still allocated higher weight than less efficient circuits.
\end{intuition}
\begin{sketch}
The assumption that every circuit produces the same logits on the training dataset implies that $\LossTrain$ is purely a function of $\sum_{i=1}^{\numCircuits} \weight{i}$. So, for $\LossTrain$, a small increase $\dweight$ to $\weight{i}$ can be balanced by a corresponding decrease $\dweight$ to some other weight $\weight{j}$.

For $\LossWD$, an increase $\dweight$ to $\weight{i}$ produces a change of approximately $\derivative{\LossWD}{\weight{i}} \cdot \dweight  = \frac{\weightDecayHyper}{\scalingExp} \left( \paramNorm{i} (\weight{i})^r \right)^{\normNum} \cdot \dweight$, where $r = \frac{1}{\scalingExp} - \frac{1}{\normNum} = \frac{\normNum - \scalingExp}{\normNum\scalingExp}$. So, an increase of $\dweight$ to $\weight{i}$ can be balanced by a decrease of $\left(\frac{\paramNorm{i}(\weight{i})^r}{\paramNorm{j}(\weight{j})^r}\right)^\normNum \dweight$ to some other weight $\weight{j}$. The two cases correspond to $r \leq 0$ and $r > 0$ respectively.

\paragraph{Case 1: $r \leq 0$.} Consider $i, j$ with $\paramNorm{j} > \paramNorm{i}$. The optimal weights must satisfy $\weight{i}^* \geq \weight{j}^*$ (else you could swap $\weight{i}^*$ and $\weight{j}^*$ to decrease loss). But then $\weight{j}^*$ must be zero: if not, we could increase $\weight{i}^*$ by $\dweight$ and decrease $\weight{j}^*$ by $\dweight$, which keeps $\LossXEnt$ constant and decreases $\LossWD$ (since $\paramNorm{i}(\weight{i}^*)^r < \paramNorm{j}(\weight{j}^*)^r$).

\paragraph{Case 2: $r > 0$.} Consider $i, j$ with $\paramNorm{j} > \paramNorm{i}$. As before we must have $\weight{i}^* \geq \weight{j}^*$. But now $\weight{j}^*$ must \emph{not} be zero: otherwise we could increase $\weight{j}^*$ by $\dweight$ and decrease $\weight{i}^*$ by $\dweight$ to keep $\LossXEnt$ constant and decrease $\LossWD$, since $\paramNorm{j}(\weight{j}^*)^r = 0 < \paramNorm{i}(\weight{i}^*)^r$. The balance occurs when $\paramNorm{j}(\weight{j}^*)^r = \paramNorm{i}(\weight{i}^*)^r$, which means $\weight{i}^* \propto \paramNorm{i}^{-1/r}$.
\end{sketch}
\begin{proof}
First, notice that our conclusions trivially hold for $\weightVec^* = \vec{0}$ (which can be a minimum if e.g. the circuits are worse than random). Thus for the rest of the proof we will assume that at least one weight is non-zero.

In addition, $\Loss \rightarrow \infty$ whenever any $w_i \rightarrow \infty$ (because $\LossTrain \geq B$ and $\LossWD \rightarrow \infty$ as any one $w_i \rightarrow \infty$). Thus, any global minimum must have finite $\vec{w}$.

Notice that, since the circuit logits are independent of $i$, we have $\classifier = \left(\sum_i \weight{i}\right) f$, and so $\LossTrain(\weightVec)$ is purely a function of the sum of weights $\sum_{i=1}^\numCircuits \weight{i}$, and the overall loss can be written as:
\begin{equation*}
    \Loss(\weightVec) = \LossTrain\left(\sum\limits_{i=1}^\numCircuits \weight{i}\right) + \frac{\weightDecayHyper}{\normNum} \sum\limits_{i=1}^\numCircuits ((\weight{i})^{\frac{1}{\scalingExp}}\paramNorm{i})^{\normNum}
\end{equation*}
We will now consider each case in order.

\paragraph{Case 1: $\scalingExp \geq \normNum$.} Assume towards contradiction that there is a global minimum $\weightVec^*$ where $\weight{j}^* > 0$ for some circuit $\circuit{j}$ with non-minimal $\paramNorm{j}$. Let $\circuit{i}$ be a circuit with minimal $\paramNorm{i}$ (so that $\paramNorm{i} < \paramNorm{j}$), and let its weight be $\weight{i}^*$.

Consider an alternate weight assignment $\weightVec'$ that is identical to $\weightVec^*$ except that $\weight{j}' = 0$ and $\weight{i}' = \weight{i}^* + \weight{j}^*$. Clearly $\sum_i \weight{i}^* = \sum_i \weight{i}'$, and so $\LossTrain(\weightVec^*) = \LossTrain(\weightVec')$. Thus, we have:
\begin{align*}
    &\Loss(\weightVec^*) - \Loss(\weightVec') \\
    &= \left( \frac{\weightDecayHyper}{\normNum} \sum\limits_{m=1}^\numCircuits ((\weight{m}^*)^{\frac{1}{\scalingExp}}\paramNorm{m})^\normNum \right) - \left( \frac{\weightDecayHyper}{\normNum} \sum\limits_{m=1}^\numCircuits ((\weight{m}')^{\frac{1}{\scalingExp}}\paramNorm{m})^\normNum \right) \\
    &= \frac{\weightDecayHyper}{\normNum} \left( (\weight{i}^*)^{\frac{\normNum}{\scalingExp}}\paramNorm{i}^\normNum + (\weight{j}^*)^{\frac{\normNum}{\scalingExp}}\paramNorm{j}^\normNum - (\weight{i}')^{\frac{\normNum}{\scalingExp}}\paramNorm{i}^\normNum \right) \\
    &> \frac{\weightDecayHyper}{\normNum} \paramNorm{i}^\normNum \left( (\weight{i}^*)^{\frac{\normNum}{\scalingExp}} + (\weight{j}^*)^{\frac{\normNum}{\scalingExp}} - (\weight{i}')^{\frac{\normNum}{\scalingExp}} \right) & \text{since } \paramNorm{j} > \paramNorm{i} \\
    &= \frac{\weightDecayHyper}{\normNum} \paramNorm{i}^\normNum \left( (\weight{i}^*)^{\frac{\normNum}{\scalingExp}} + (\weight{j}^*)^{\frac{\normNum}{\scalingExp}} - (\weight{i}^* + \weight{j}^*)^{\frac{\normNum}{\scalingExp}} \right) & \text{definition of } \weight{i}' \\
    &\geq \frac{\weightDecayHyper}{\normNum} \paramNorm{i}^\normNum \left( (\weight{i}^*)^{\frac{\normNum}{\scalingExp}} + (\weight{j}^*)^{\frac{\normNum}{\scalingExp}} - \left((\weight{i}^*)^{\frac{\normNum}{\scalingExp}} + (\weight{j}^*)^{\frac{\normNum}{\scalingExp}}\right) \right) & \text{using Lemma~\ref{lemma:sum-of-powers} since } 0 < \frac{\normNum}{\scalingExp} \leq 1 \\
    &= 0
\end{align*}

Thus we have $\Loss(\weightVec^*) > \Loss(\weightVec')$, contradicting our assumption that $\weightVec^*$ is a global minimum of $\Loss$. This completes the proof for the case that $\scalingExp \geq \normNum$.

\paragraph{Case 2: $\scalingExp < \normNum$.} First, we will show that all weights are non-zero at a global minimum (excluding the case where $\weightVec^* = \vec{0}$, discussed at the beginning of the proof). Assume towards contradiction that there is a global minimum $\weightVec^*$ with $\weight{j}^* = 0$ for some $j$. Choose some arbitrary circuit $\circuit{i}$ with nonzero weight $\weight{i}^*$.

Choose some $\epsilon_1 > 0$ satisfying $\epsilon_1 < \frac{\normNum}{2\scalingExp}(\weight{i}^*)^{\frac{\normNum}{\scalingExp}-1}$. By applying Lemma~\ref{lemma:power-derivative} with $x = \weight{i}^*, c = \epsilon_1, r = \frac{\normNum}{\scalingExp}$, we can get some $\delta > 0$ such that for any $\epsilon < \delta$ we have $(\weight{i}^*)^{\frac{\normNum}{\scalingExp}} - (\weight{i}^* - \epsilon)^{\frac{\normNum}{\scalingExp}} > \delta (\frac{\normNum}{\scalingExp} (\weight{i}^*)^{\frac{\normNum}{\scalingExp}-1} - \epsilon_1)$.

Choose some $\epsilon_2 > 0$ satisfying $\epsilon_2 < \min(\weight{i}^*, \delta, \left[ \frac{\normNum}{2\scalingExp}(\weight{i}^*)^{\frac{\normNum}{\scalingExp}-1} \frac{\paramNorm{i}^{\normNum}}{\paramNorm{j}^{\normNum}} \right]^{\frac{1}{\frac{\normNum}{\scalingExp}-1}})$. Consider an alternate weight assignment defined $\weightVec'$ that is identical to $\weightVec^*$ except that $\weight{j}' = \epsilon_2$ and $\weight{i}' = \weight{i}^* - \epsilon_2$. As in the previous case, $\LossTrain(\weightVec^*) = \LossTrain(\weightVec')$. Thus, we have:
\begin{align*}
    &\Loss(\weightVec^*) - \Loss(\weightVec') \\
    &= \frac{\weightDecayHyper}{\normNum} \left( (\weight{i}^*)^{\frac{\normNum}{\scalingExp}}\paramNorm{i}^{\normNum} - (\weight{i}^* - \epsilon_2)^{\frac{\normNum}{\scalingExp}}\paramNorm{i}^{\normNum} - \epsilon_2^{\frac{\normNum}{\scalingExp}} \paramNorm{j}^{\normNum} \right) \\
    &= \frac{\weightDecayHyper}{\normNum} \left( \paramNorm{i}^{\normNum} ((\weight{i}^*)^{\frac{\normNum}{\scalingExp}} - (\weight{i}^* - \epsilon_2)^{\frac{\normNum}{\scalingExp}}) - \epsilon_2^{\frac{\normNum}{\scalingExp}} \paramNorm{j}^{\normNum} \right) \\
    &> \frac{\weightDecayHyper}{\normNum} \left( \paramNorm{i}^{\normNum} \delta (\frac{\normNum}{\scalingExp}(\weight{i}^*)^{\frac{\normNum}{\scalingExp}-1} - \epsilon_1) - \epsilon_2^{\frac{\normNum}{\scalingExp}} \paramNorm{j}^{\normNum} \right) &\text{application of Lemma~\ref{lemma:power-derivative} discussed above} \\
    &> \frac{\weightDecayHyper}{\normNum} \left( \paramNorm{i}^{\normNum} \delta (\frac{\normNum}{\scalingExp}(\weight{i}^*)^{\frac{\normNum}{\scalingExp}-1} - \frac{\normNum}{2\scalingExp}(\weight{i}^*)^{\frac{\normNum}{\scalingExp}-1}) - \epsilon_2^{\frac{\normNum}{\scalingExp}} \paramNorm{j}^{\normNum} \right) &\text{we chose } \epsilon_1 < \frac{\normNum}{2\scalingExp}(\weight{i}^*)^{\frac{\normNum}{\scalingExp}-1}\\
    &> \frac{\weightDecayHyper}{\normNum} \left( \paramNorm{i}^{\normNum} \epsilon_2 \frac{\normNum}{2\scalingExp}(\weight{i}^*)^{\frac{\normNum}{\scalingExp}-1} - \epsilon_2^{\frac{\normNum}{\scalingExp}} \paramNorm{j}^{\normNum} \right) &\text{we chose } \epsilon_2 < \delta\\
    &= \frac{\weightDecayHyper\epsilon_2}{\normNum} \left( \frac{\normNum}{2\scalingExp}(\weight{i}^*)^{\frac{\normNum}{\scalingExp}-1}\paramNorm{i}^{\normNum} - \epsilon_2^{\frac{\normNum}{\scalingExp}-1} \paramNorm{j}^{\normNum} \right)\\
    &> \frac{\weightDecayHyper\epsilon_2}{\normNum} \left( \frac{\normNum}{2\scalingExp}(\weight{i}^*)^{\frac{\normNum}{\scalingExp}-1}\paramNorm{i}^{\normNum} - \frac{\normNum}{2\scalingExp}(\weight{i}^*)^{\frac{\normNum}{\scalingExp}-1} \frac{\paramNorm{i}^{\normNum}}{\paramNorm{j}^{\normNum}} \paramNorm{j}^{\normNum} \right) & \text{we chose } \epsilon_2 < \left[ \frac{\normNum}{2\scalingExp}(\weight{i}^*)^{\frac{\normNum}{\scalingExp}-1} \frac{\paramNorm{i}^{\normNum}}{\paramNorm{j}^{\normNum}} \right]^{\frac{1}{\frac{\normNum}{\scalingExp}-1}}) \\
    &= 0
\end{align*}

Note that in the last step, we rely on the fact that $\scalingExp < \normNum$: this lets us use an upper bound on $\epsilon_2$ to get an upper bound on $\epsilon_2^{\frac{\normNum}{\scalingExp}-1}$, and so a lower bound on the overall expression. 

Thus we have $\Loss(\weightVec^*) > \Loss(\weightVec')$, contradicting our assumption that $\weightVec^*$ is a global minimum of $\Loss$. So, for all $i$ we have $\weight{i} > 0$.

In addition, as $\weight{i} \rightarrow \infty$ we have $\Loss(\weightVec) \rightarrow \infty$, so $\weightVec^*$ cannot be at the boundaries, and instead lies in the interior. Since $\normNum > \scalingExp$, $\Loss(\weightVec)$ is differentiable everywhere. Thus, we can conclude that its gradient at $\weightVec^*$ is zero:
\begin{align*}
    \derivative{\Loss}{\weight{i}} &= 0 \\
    \derivative{\LossTrain}{\weight{i}} + \frac{\weightDecayHyper\paramNorm{i}^{\normNum}}{\scalingExp} (\weight{i}^*)^{\frac{\normNum}{\scalingExp}-1} &= 0
    \\
    \paramNorm{i}^{\normNum} (\weight{i}^*)^{\frac{\normNum - \scalingExp}{\scalingExp}} &= - \frac{\scalingExp}{\weightDecayHyper} \derivative{\LossTrain}{\weight{i}}
    \\
    \weight{i}^* \paramNorm{i}^{\frac{\normNum\scalingExp}{\normNum-\scalingExp}} &= \left( -\frac{\scalingExp}{\weightDecayHyper} \derivative{\LossTrain}{\weight{i}} \right)^{\frac{\scalingExp}{\normNum-\scalingExp}}
\end{align*}

Since $\LossTrain(\weightVec)$ is a function of $\sum\limits_{j=1}^\numCircuits \weight{j}$, we can conclude that $\derivative{\LossTrain}{\weight{i}} = \derivative{\LossTrain}{\sum_j \weight{j}} \cdot \derivative{\sum_j \weight{j}}{\weight{i}} = \derivative{\LossTrain}{\sum_j \weight{j}}$, which is independent of $i$. So the right hand side of the equation is independent of $i$, allowing us to conclude that $\weight{i}^* \propto \paramNorm{i}^{-\frac{\normNum\scalingExp}{\normNum-\scalingExp}}$.

\end{proof}

\end{document}